\documentclass[11pt]{article} 

\usepackage{bbm}
\usepackage{amsmath}
\usepackage{amssymb}
\usepackage[framemethod=TikZ]{mdframed}
\usepackage{mathtools}
\usepackage{amsthm}
\usepackage{natbib}
\usepackage{letltxmacro}
\usepackage{authblk}
\usepackage{centernot}
\usepackage{thmtools}
\usepackage{xcolor}
\usepackage{subcaption}

\usepackage[letterpaper,top=2cm,bottom=2cm,left=3cm,right=3cm,marginparwidth=1.75cm]{geometry}

\newcount\Comments
\usepackage{color}
\definecolor{darkgreen}{rgb}{0,0.5,0}
\definecolor{purple}{rgb}{1,0,1}
\newcommand{\kibitz}[2]{\ifnum\Comments=1\textcolor{#1}{#2}\fi}

\usepackage{enumitem}
\setlist[itemize]{noitemsep, topsep=2pt}
\setlist[enumerate]{noitemsep, topsep=2pt}

\title{On Generation in Metric Spaces}
\usepackage{times}

\author{Jiaxun Li, Vinod Raman, Ambuj Tewari}
\affil{Department of Statistics, University of Michigan}
\affil{\texttt{\{jasonli, vkraman, tewaria\}@umich.edu}}
\date{\today}

\usepackage{algorithm}
\usepackage{dsfont}

\PassOptionsToPackage{algo2e,ruled, linesnumbered}{algorithm2e}
\RequirePackage{algorithm2e}

\newcommand{\Gcal}{\mathcal{G}}
\newcommand{\Hcal}{\mathcal{H}}

\newcommand{\Ocal}{\mathcal{O}}

\newcommand{\Xcal}{\mathcal{X}}

\newcommand{\vep}{\varepsilon}

\DeclareMathOperator*{\argmax}{arg\,max}

\newcommand{\naturals}{\mathbb{N}}

\newtheorem{theorem}{Theorem}[section]
\newtheorem{proposition}[theorem]{Proposition}
\newtheorem{definition}{Definition}[section]
\newtheorem{lemma}[theorem]{Lemma}

\newtheorem{corollary}[theorem]{Corollary}
\newtheorem{example}[theorem]{Example}

\newtheorem{remark}{Remark}[section]
\newtheorem*{theorem*}{Theorem}

\begin{document}

\maketitle

\begin{abstract}%

We study generation in separable metric instance spaces. We extend the generation framework from \citet{kleinberg2024language} beyond countable domains by defining novelty through metric separation and allowing asymmetric novelty parameters for the adversary and the generator. We introduce the $(\varepsilon_a,\varepsilon_g)$-closure dimension, a scale-sensitive analogue of closure dimension, which yields characterizations of uniform and non-uniform generatability. We also show that the union of two uniformly generatable classes can fail to be generatable in the limit, even on the real line. Along the way, we identify a sharp geometric contrast. Namely, in doubling spaces, including all finite-dimensional normed spaces, generatability is stable across novelty scales and invariant under equivalent metrics. In general metric spaces, however, generatability can be highly scale-sensitive and metric-dependent; even in the natural infinite-dimensional Hilbert space $\ell^2$, all notions of generation may fail abruptly as the novelty parameters vary.

\end{abstract}

\section{Introduction}
Generative models are trained on massive amounts of data from which they learn to generate new, previously unseen, well-formed objects such as text, images, code, and molecules. A clean, minimal framework has emerged in recent work \citep{kleinberg2024language,li2024generationlenslearningtheory} that captures the essence of generation. In this model, the learner receives a sequence of objects from an adversary drawn from a set in a known set collection. The learner has to generate novel valid objects from the true set that is initially unknown to the learner. This framework is inspired by the seminal work of \citet{gold1967language} on language identification in the limit. In \citet{gold1967language}’s model, an adversary enumerates strings from a language in a collection. The learner’s task is to identify the true language from only the positive examples shown by the adversary. \citet{kleinberg2024language} changed the identification requirement to generation: the learner has to generate novel, unseen objects from the true set.

This simple change, from identification to generation, completely changes the nature of provable guarantees. On the one hand, since the seminal characterization by \citet{angluin1979finding} of language collections identifiable in the \citet{gold1967language} model, we know that very few natural collections are identifiable in the limit. On the other hand, \citet{kleinberg2024language} showed that \emph{any} countable collection is generatable in the limit. 
They also connected their theoretical findings to empirical phenomena such as under-generation (mode collapse) and over-generation (hallucination). Their elegant framework and surprising results have led to a flurry of work on extending their basic framework in various interesting directions (see Sec.~\ref{sec:related}).

However, most work in this area assumes a countable example space with a discrete $0$-$1$ metric. This implies that the generated object is considered novel if it is distinct from previously seen objects. This notion is often too weak: in image generation, changing a single pixel of an image will not necessarily produce a genuinely new image, and in text generation, replacing one word in a sentence will not necessarily produce a genuinely new sentence.

In this paper we initiate a theoretical study of \emph{generation in metric spaces}. A metric space structure is general enough to encompass not only traditional modalities like text and images but also emerging applications such as SDE solutions over a spatial domain \citep{nerualoperator} and protein generation in 3D space that respects
translational and rotational symmetries~\citep{pmlr-v162-hoogeboom22a}.

Motivated by practice, we operationalize novelty by requiring that the distance of generated objects to seen objects must lie above a given threshold $\vep_g$. For example, in protein structure generation, a generated structure is often classified as novel only if its maximum structural similarity to all training structures falls below a fixed threshold, corresponding to a distance-based novelty notion that is relative to the training set \citep{linproteingen}.

Furthermore, empirical studies of generative model generalization have observed that generated samples may lie closer to the training set than independently held-out test samples in a feature space, reflecting a separation between the scale at which training data are revealed and the scale at which novelty is required of generated outputs \citep{jiralerspong2023feature}. Our framework makes this distinction explicit by allowing separate novelty parameters $\vep_a$, $\vep_g$: one for the adversary, and one for the generator.


\subsection{Main Contributions}
Our main contributions are as follows:
\begin{itemize}
    \item
    We extend the language generation model to separable metric instance spaces $\Xcal$, while retaining binary hypothesis classes $\mathcal{H} \subseteq \{0,1\}^\Xcal$. We introduce a notion of novelty based on metric separation rather than mere distinctness, and define two independent novelty parameters $\vep_a,\vep_g$, for the adversary and the generator, respectively. This leads to a strictly richer and nontrivial generalization of the countable setting.

    \item 
    We introduce the $(\vep_a,\vep_g)$-closure dimension, a metric analogue of closure dimension that explicitly depends on scale. Using this notion, we obtain complete characterizations of uniform and non-uniform generatability. We also show that the union of two uniformly generatable classes can fail to be generatable in the limit, even on the real line. Moreover, we generalize the algorithmic framework from the countable case to the metric space setting under appropriate oracle access.

    \item
    In doubling spaces, including all finite-dimensional normed spaces, we show that generatability behaves in a well-controlled manner: uniform and non-uniform generatability is invariant to the choice of novelty scales and preserved under equivalent metrics. 
    \item 
   Beyond doubling spaces, we demonstrate sharp, non-intuitive behaviors. Generatability can be highly sensitive to novelty parameters and metric choice. Our constructions show that all notions of generation can fail abruptly as scales change even in the simple infinite-dimensional space $\ell^2$, highlighting a qualitative gap between finite and infinite-dimensional spaces.

\end{itemize}

Extending generation to metric spaces presents several challenges.
First, unlike in the countable setting, novelty can no longer be defined by distinctness alone and must reflect geometric separation.
This motivates asymmetric novelty parameters for the adversary and the generator.
Second, several characterization results in the countable setting rely explicitly on countability, for example through enumeration arguments, and do not extend to general metric spaces.
Finally, due to the asymmetry between the adversary’s and generator’s novelty parameters, constructing hypothesis classes and counterexamples requires careful consideration of their relative scales.

\subsection{Related Work}
\label{sec:related}

Our work builds on the framework of language generation in the limit, introduced by \citet{kleinberg2024language}. Since then, subsequent work has investigated a range of questions surrounding this language generation model. Here, we discuss the most relevant prior works.
\paragraph{Uniform and Non-Uniform Generation}
\citet{li2024generationlenslearningtheory} introduced stronger notions of generation: uniform and non-uniform generation, and provided characterization for uniform and non-uniform generation. \citet{charikar2024exploringfacetslanguagegeneration} also, independently studied non-uniform generation and showed that all countable collections can be non-uniformly generated. \citet{charikar2025pareto} proposed an almost Pareto-optimal algorithm for non-uniform language generation. \citet{hanneke2025unionclosednesslanguagegeneration} studied unions of generatable classes and solved two open problems posed by \citet{li2024generationlenslearningtheory}.

\paragraph{Extensions of Language Generation}
Several extensions of language generation have been proposed. \citet{raman2025generation}, \citet{mehrotra2025language} and \citet{bai2026language} investigated language generation with contamination, in which the adversary is allowed to omit instances or generate instances outside the language.  \citet{peale2025representative} studied representative generation, in which the generator needs to output a distribution that is representative of the seen data.  \citet{anastasopoulos2026safe} investigated safe language generation, in which there exists a harmful language that the generator needs to avoid. 
\paragraph{Other Related Work}
Several follow-up studies have proposed and studied different definitions of breadth in the language generation model, and showed that requiring the generator to generate diverse instances from the target language makes generation significantly harder \citep{kalavasis2024characterizations,kalavasis2024limits,kleinberg2025density,kleinberg2025language}. 

\citet{alves2021languagelearnabilitylimitgeneral} study language identification in metric spaces by revisiting classical identification in the limit and allowing distances between languages to be measured by a metric. In contrast, our work takes a fundamentally different perspective: we study \emph{generation} rather than identification, and impose metric structure on the \emph{instance space} rather than on the space of languages.


\section{Preliminaries} \label{sec:prelim}
\subsection{Example Space and Hypothesis Class}
Throughout the paper, we let $(\Xcal,\rho)$ be a \emph{separable metric space}, meaning that $\Xcal$ admits a countable \emph{dense} subset with respect to $\rho$. Let $\Hcal \subseteq \{0, 1\}^{\Xcal}$ be a binary hypothesis class. For a hypothesis $h \in \Hcal$, let $\text{supp}(h) := \{x \in \Xcal: h(x) = 1\}$, that is, its collection of positive examples. 
For any class $\Hcal$ and a finite sequence of examples $x_1, \dots, x_n$, define the version space as:
$$\Hcal(x_1, \dots, x_n) := \{h \in \Hcal: \{x_1, \dots, x_n\} \subseteq \operatorname{supp}(h)\}.$$
For any class $\Hcal$, define its induced closure operator $\langle \cdot \rangle_{\Hcal}$ as:

$$\langle x_1,\dots,x_n \rangle_{\Hcal} := \begin{cases}
			\bigcap_{h \in \Hcal(x_{1:n})} \operatorname{supp}(h), & \text{if $|\Hcal(x_{1:n})| \geq 1$}\\
            \bot, & \text{if $|\Hcal(x_{1:n})| = 0$}
		 \end{cases}.$$

\noindent Intuitively, $\langle x_1,\dots,x_n\rangle_\Hcal$ captures the set that must belong to the target hypothesis given consistency with the observed examples. We will also need several other definitions about metric spaces.
\begin{definition}[Closed Ball]\label{def:closeballs}
Let $(\Xcal,\rho)$ be a metric space. For a point $x\in\Xcal$ or a set $A\subseteq \Xcal$, we denote by $B(x,\vep)$ or $B(A,\vep)$ the closed $\vep$-balls centered at $x$ or $A$, respectively, defined as
\[
B(x,\vep) := \{ y \in \Xcal : \rho(x,y) \le \vep \}, 
\qquad
B(A,\vep) := \{ y \in \Xcal : \inf_{x\in A} \rho(x,y) \le \vep \}.
\]
\end{definition}

\begin{definition}[Covering Number]
    Let $(\Xcal,\rho)$ be a metric space and $\vep>0$.  
    An $\vep$-covering of a set $A\subseteq \Xcal$ is a finite set $\{\theta_1,\dots,\theta_M\}\subset \Xcal$ such that
    $A\subset B(\{\theta_i\}_{i=1}^M,\vep)$. The $\vep$-covering number $N(\vep;A,\rho)$ is defined as
    \[
        N(\vep;A,\rho)
        :=
        \min\left\{M\in\mathbb N:\exists\ \vep\text{-covering of }A \text{ with }M\text{ points}\right\}.
    \]
    If no such finite covering exists, we write $N(\vep;A,\rho)=\infty$.
\end{definition}

To prevent the adversary from exhausting all positive examples and thus making the task of generation impossible, we make the following assumption for the hypothesis set. This assumption extends a similar condition used in previous work on generation in the countable example space setting (see Definition \ref{def:countable-uus}).

\begin{definition}[$r$-Uniformly Unbounded Support ($r$-UUS)]
     Let $(\Xcal,\rho)$ be a metric space and let $r>0$. A hypothesis class $\Hcal\subseteq\{0,1\}^{\Xcal}$ satisfies the \emph{$r$-Uniformly Unbounded Support ($r$-UUS)} property if every $h\in\Hcal$ has nonempty support and, for every sequence $(x_i)_{i\ge1}$ in $\operatorname{supp}(h)$ satisfying
     \[
     \operatorname{supp}(h)\subseteq B(\{x_i:i\ge1\},r),
     \]
     we have
     \[
     \lim_{t\to\infty}N(r;\{x_1,\dots,x_t\},\rho)=\infty.
     \]
\end{definition}
The $r$-UUS property satisfies the following monotonicity result, which we prove in Appendix~\ref{sec:appendix-other}.
\begin{lemma}\label{lem:uusdiffeps}
    Let $\Hcal\subseteq \{0,1\}^\Xcal$ be any class satisfying the $r$-UUS property for an $r>0$. Then for any positive $\delta<r$, $\Hcal$ satisfies the $\delta$-UUS property. 
\end{lemma}


\subsection{Generatability}
Consider the following two-player game for a hypothesis set $\Hcal$. At the start, the adversary picks a hypothesis $h \in \Hcal$ and does not reveal it to the learner. The game then proceeds over rounds $t = 1, 2, \dots$. In each round $t \in \naturals$, the adversary reveals a \emph{new} example $x_t\in \operatorname{supp}(h)$. The generator, after observing enough examples of $x_1, \dots, x_t$, must output some \emph{new} example inside $\operatorname{supp}(h)$. Crucially, the generator \emph{never} observes its loss as it does not know $h$. The goal of the generator is to eventually perfectly generate new, positive examples $x \in \operatorname{supp}(h)$. To make this formal, we first define a generator as a mapping from a finite sequence of examples to a new example.

\begin{definition}[Generator] \label{def:generator} A generator is a map $\Gcal: \Xcal^{\star} \rightarrow \Xcal$ that takes a finite sequence of examples $x_1, x_2, \dots$ and outputs a new example $x$. 
\end{definition}

The challenge of extending the framework to metric spaces lies in the definition of \emph{novelty}. In a countable example space, novelty can simply be defined as \emph{distinct instances}. In a metric space, however, distinctness is not enough: one does not expect an instance that is very close to the previous ones to be a \emph{new} instance.

As a result, we define a \emph{new} instance to be $\vep$ far away from the previous instances. To make the definition general, the $\vep$ is allowed to be different on the adversary and the generator sides. That is, we define two novelty parameters for the adversary and the generator, denoted as $\vep_a$ and $\vep_g$, respectively.
\begin{definition}[Generatability in the Limit] \label{def:geninlim} Let $\Hcal \subseteq \{0, 1\}^{\Xcal}$  be any hypothesis class satisfying the $r$-$\operatorname{UUS}$ property for some $r>0$ and let
$0<\varepsilon_a,\varepsilon_g \le r$. If there exists a generator $\Gcal$ such that for every $h \in \Hcal$, and any sequence $\{x_i\}_{i=1}^\infty\subseteq \operatorname{supp}(h)$ such that $\operatorname{supp}(h)\subseteq B(\{x_i\}_{i=1}^\infty,\vep_a)$ (i.e., the adversary eventually reveals an $\vep_a$-cover of the support), there exists a  $t^{\star} \in \naturals$ such  that $\Gcal(x_{1:s}) \in \operatorname{supp}(h) \setminus  B(\{x_i\}_{i=1}^s,\vep_g)$ for all $s \geq t^{\star}$, then we say $\Hcal$ is \emph{($\vep_a,\vep_g$)-generatable in the limit}.
\end{definition}

\begin{remark}
    The assumption $\operatorname{supp}(h)\subseteq B(\{x_i\}_{i=1}^\infty,\vep_a)$ is an analogous to the enumeration assumption in the countable example space setting (see Definition~\ref{def:countable-limit-gen}). It requires the adversary to eventually show the whole structure of the hypothesis set to the generator.
\end{remark}

Our definition above is intentionally general and does not specify the relative magnitudes of $\vep_a$ and $\vep_g$. This is because the relative magnitudes of novelty on the adversary and generator sides are application dependent. For example, consider authors preparing a paper submission to a conference.
An experienced author may read several papers that are closely related in topics (corresponding to a small $\vep_a$ for adversary) yet generate a paper that is substantially different, for example, by introducing a new technique or applying known methods to a new problem. In this case the generator's $\vep_g$ is larger than than the adversary's $\vep_a$.
Alternatively, a beginning researcher may read a very diverse set of papers spanning many topics (a large $\vep_a$), yet generate a paper closely aligned with one of them, such as by extending an existing line of work. Here, the generator's $\vep_g$ is smaller than the adversary's $\vep_a$. Our definition of generatability in the limit is designed to encompass both behaviors.

Next, we define uniform generation, which requires a common covering-number threshold independent of both the target hypothesis and the positive sequence. The round at which this threshold is reached may depend on the sequence.
\begin{definition}[Uniform Generatability] \label{def:unifgen} Let $\Hcal \subseteq \{0, 1\}^{\Xcal}$  be any hypothesis class satisfying the $r$-$\operatorname{UUS}$ property for some $r>0$ and let
$0<\vep_a,\vep_g \le r$. If there exists a generator $\Gcal$ and $d^{\star} \in \naturals$, such that for every $h \in \Hcal$ and any sequence $\{x_i\}_{i=1}^\infty\subseteq\operatorname{supp}(h)$, if there exists $t$ such that $N(\vep_a;\{x_i\}_{i=1}^t,\rho)\ge d^\star$, then $\Gcal(x_{1:s}) \in  \operatorname{supp}(h) \setminus  B(\{x_i\}_{i=1}^s,\vep_g)$ for all $s \geq t$, then we say $\Hcal$ is \emph{($\vep_a,\vep_g$)-uniformly generatable}.
\end{definition}

For non-uniform generation, the covering-number threshold may depend on the target hypothesis, but not on the positive sequence.

\begin{definition}[Non-Uniform Generatability]\label{def:nonunifgen} 
Let $\Hcal \subseteq \{0, 1\}^{\Xcal}$ be any hypothesis class satisfying the $r$-$\operatorname{UUS}$ property for some $r>0$ and let
$0<\vep_a,\vep_g \le r$. If there exists a generator $\Gcal$ such that for every $h \in \Hcal$ there exists a $d_h\in\naturals$, such that for any sequence $\{x_i\}_{i=1}^\infty\subseteq \operatorname{supp}(h)$, if there exists~$t$ such that $N(\vep_a;\{x_i\}_{i=1}^t,\rho)\ge d_h$, then $\Gcal(x_{1:s}) \in  \operatorname{supp}(h) \setminus  B(\{x_i\}_{i=1}^s,\vep_g)$ for all $s \geq t$, then we say $\Hcal$ is \emph{($\vep_a,\vep_g$)-non-uniformly generatable}.
\end{definition}


It is clear that our framework strictly generalizes the countable example space setting considered in \citet{li2024generationlenslearningtheory}. This is achieved by equipping a countable example space with the discrete $\{0,1\}$-metric. We include a formal proposition in Appendix~\ref{sec:appendix-other} for completeness sake.

\section{Characterizations of Generatability} \label{sec:chargen}

In this section, we characterize hypothesis classes that are uniformly and non-uniformly generatable, and we show that the union of two uniformly generatable classes can fail to be generatable in the limit. We show that, despite the additional metric structure, the resulting characterizations closely parallel those obtained in the countable example space setting (see Appendix~\ref{sec:appendix_countable} for more details). Consequently, we state the main results in this section and defer the proofs to Appendix~\ref{sec:appendix-proof-chargen}. In Appendix~\ref{sec:appendix-algo}, we also have a discussion about the computability of generation.  Although these characterizations are straightforward, generation in metric spaces exhibits several counterintuitive phenomena that do not arise in the countable example space setting. We explore these phenomena in more detail in Section~\ref{sec:diffeps}. 

We first extend the techniques from \citet{li2024generationlenslearningtheory} to account for metric space structure. As a result, we first define a metric analog of the Closure dimension originally defined for countable example space (see Definition~\ref{def:countable-closure-dim}).

\begin{definition}[$(\vep_a,\vep_g)$-closure dimension]
\label{def:closure-dim}
The \emph{$(\vep_a,\vep_g)$-closure dimension} of $\Hcal$, denoted
$\operatorname{C}_{\vep_a,\vep_g}(\Hcal)$, is the largest
$d\in\mathbb{N}$ such that there exist $h\in\Hcal$ and
$\{x_i\}_{i=1}^t\subseteq\operatorname{supp}(h)$ satisfying
\[
N(\vep_a;\{x_i\}_{i=1}^t,\rho)=d,
\]
and a finite set $F\subseteq\langle x_1,\dots,x_t\rangle_{\Hcal}$
such that
\[
\langle x_1,\dots,x_t\rangle_{\Hcal}\subseteq B(F,\vep_g).
\]
If this holds for arbitrarily large $d\in\mathbb{N}$, we set
$\operatorname{C}_{\vep_a,\vep_g}(\Hcal)=\infty$.
If no such sequence exists for any $d\in\mathbb{N}$, we set
$\operatorname{C}_{\vep_a,\vep_g}(\Hcal)=0$.
\end{definition}

Unlike the Closure dimension, the $(\vep_a,\vep_g)$-Closure dimension is \emph{scale-sensitive}. In Section~\ref{sec:diffeps} we will further explore its property for different scales of $\vep_a$ and $\vep_g$. Our first theorem uses this dimension to provide a complete characterization of which classes are uniform generatable.

\begin{theorem}[Characterization of Uniform Generatability] \label{thm:unifgen} Let $\Hcal \subseteq \{0, 1\}^{\Xcal}$  be any hypothesis class satisfying the $r$-$\operatorname{UUS}$ property for some $r>0$ and let
$0< \vep_a, \vep_g \le r$. Then, $\Hcal$ is $(\vep_a,\vep_g)$-uniformly generatable if and only if $\operatorname{C}_{ \vep_a, \vep_g} (\Hcal) < \infty$.

\end{theorem}

Similarly to the countable example space setting, for every finite hypothesis class $\mathcal H$ satisfying the $r$-UUS property with $ \vep_a \ge  \vep_g$, we have $\operatorname{C}_{ \vep_a,\vep_g}(\mathcal H) < \infty$. In contrast, when $ \vep_a <  \vep_g$, there exists a finite hypothesis class that is not non-uniformly generatable. Intuitively, if the generator is required to produce outputs that are more novel than those revealed by the adversary, the generation task may become impossible. Further discussion is provided in Appendix~\ref{sec:appendix-proof-chargen}.

\begin{corollary}
    [Finite Classes are Uniformly Generatable]\label{cor:genfinite}
Let $\Hcal \subseteq \{0, 1\}^{\Xcal}$  be any finite hypothesis class satisfying the $r$-$\operatorname{UUS}$ property for some $r>0$ and let
$0< \vep_a, \vep_g \le r$ such that $\vep_g\le \vep_a$. Then, $\Hcal$ is $(\vep_a,\vep_g)$-uniformly generatable. 
\end{corollary}

We next move to characterize non-uniform generatability. We show that our characterization of uniform generatability leads to a characterization of non-uniform generatability.

\begin{theorem}[Characterization of Non-uniform Generatability] \label{thm:nonunifgen} Let $\Hcal \subseteq \{0, 1\}^{\Xcal}$  be any hypothesis class with the $r$-$\operatorname{UUS}$ property for some $r>0$ and let
$0< \vep_a, \vep_g \le r$. The following statements are equivalent. 

\begin{itemize}
\item[(i)] $\Hcal$ is $(\vep_a,\vep_g)$-non-uniformly generatable.
\item[(ii)] There exists a \emph{non-decreasing} sequence of classes $\Hcal_1 \subseteq \Hcal_2 \subseteq \dots$ such that $\Hcal =\bigcup_{i\in\naturals} \Hcal_i$ and $\operatorname{C}_{ \vep_a, \vep_g} (\Hcal_n)<\infty$ for every $n\in\naturals.$
\end{itemize}
\end{theorem}

An immediate consequence of combing Corollary~\ref{cor:genfinite} and Theorem~\ref{thm:nonunifgen} is the following corollary.

\begin{corollary}[Countable Classes are Non-uniformly Generatable] \label{cor:countnonunif}
Let $\Hcal \subseteq \{0, 1\}^{\Xcal}$  be any countable hypothesis class satisfying the $r$-$\operatorname{UUS}$ property for some $r>0$ and let
$0<\vep_a, \vep_g \le r$ such that $\vep_g\le\vep_a$. Then, $\Hcal$ is $(\vep_a,\vep_g)$-non-uniformly generatable. 
\end{corollary}

We next show that uniform generatability of the component classes does not ensure generation in the limit for their union. This obstruction occurs even on the real line. 
\begin{theorem}[Failure of Finite-Union Closure]\label{thm:geninlim}
Let $r>0$ and $0<\vep_a,\vep_g\le r$. On $\mathbb R$ with its usual metric, there exist hypothesis classes $\Hcal_1,\Hcal_2\subseteq\{0,1\}^{\mathbb R}$, each satisfying $r$-UUS and each $(\vep_a,\vep_g)$-uniformly generatable, such that $\Hcal_1\cup\Hcal_2$ is not $(\vep_a,\vep_g)$-generatable in the limit.
\end{theorem}

In contrast, \citet{li2024generationlenslearningtheory} showed that, in the countable example space setting, every finite union of uniformly generatable classes is generatable in the limit. Theorem~\ref{thm:geninlim} shows that this finite-union property fails in the metric setting, even for two classes on the real line.

\section{Properties of Generatability}

\label{sec:diffeps}


Beyond the structure of the hypothesis class itself, generatability depends on two additional parameters: the novelty thresholds with respect to the adversary and the generator, denoted by $\vep_a$ and $\vep_g$, respectively. In this section, we study how generatability varies as these novelty parameters change. We also investigate generatability under different metrics defined on the same underlying space.

\subsection{Generatability is Well Behaved in Doubling Spaces}
\label{sec:doubling}

We first establish a robustness result for generatability in doubling spaces. A doubling space is a metric space in which balls at one scale can be covered by finitely many balls at any smaller scale. In such spaces, we show that, both uniform and non-uniform generatability are invariant across all choices of scale parameters $\vep_a$ and $\vep_g$, and are preserved under any change to an equivalent metric. Throughout, we say that two metrics are equivalent if they are bi-Lipschitz equivalent; see Appendix~\ref{app:equiv-metrics} for details.

\begin{definition}
\label{def:doublingspace}
A metric space $(\Xcal,\rho)$ is said to be \emph{doubling} if there exists a constant $M>0$ such that for any $x\in\Xcal$ and any $r>0$, the ball $ B(x,r)$
can be covered by at most $M$ balls of radius $r/2$.
\end{definition}

We first show a stronger result than Lemma~\ref{lem:uusdiffeps}: in doubling spaces, the $r$-UUS property is invariant to the choice of the scale parameter $r$. The proof is deferred to Appendix~\ref{sec:appendix-other} for completeness.
\begin{theorem}
    \label{thm:doubling-UUS}
    Let $(\Xcal,\rho)$ be a doubling metric space. Suppose $\Hcal \subseteq \{0,1\}^\Xcal$ is a class satisfying $r$-UUS property for some $r > 0$. Then it also satisfies the $r'$-UUS property for any $r' > 0$. In this case, we simply say that $\Hcal$ has the UUS property.
\end{theorem}

We are now ready to present robustness results for uniform and non-uniform generatability in doubling spaces.
\begin{theorem}\label{thm:doubling-eps-uniformgen}
Let $(\Xcal,\rho)$ be a doubling metric space.  
Let $\vep_a,\vep_g,\delta_a,\delta_g>0$ satisfy $0<\delta_a<\vep_a$ and $0<\delta_g<\vep_g$. 
Suppose $\Hcal\subseteq \{0,1\}^\Xcal$ satisfies the UUS property. 
Then $\Hcal$ is $(\vep_a,\vep_g)$-uniformly (resp.\ non-uniformly) generatable if and only if it is $(\delta_a,\delta_g)$-uniformly (resp.\ non-uniformly) generatable.
\end{theorem}

\begin{proof}
     By Theorem \ref{thm:unifgen} and \ref{thm:nonunifgen}, we only need to show that $\operatorname{C}_{\vep_a,\vep_g}(\Hcal)<\infty$ if and only if $\operatorname{C}_{\delta_a,\delta_g}(\Hcal)<\infty$. Since $\Xcal$ is a doubling space, there exists $C$ such that for any $x\in\Xcal$, $B(x,\vep_a)$ can be covered by the union of at most $C$ balls of radius $\delta_a$ and $B(x,\vep_g)$ can be covered by the union of at most $C$ balls of radius $\delta_g$. Then under our assumption, for any subset $A\subseteq \Xcal$,
     \[
    N(\vep_a;A,\rho)=\infty\iff N(\delta_a;A,\rho)=\infty.
    \]
    And if $N(\vep_a;A,\rho)<\infty$, we have 
    \[
    N(\vep_a;A,\rho)\le N(\delta_a;A,\rho)\le CN(\vep_a;A,\rho).
    \]
    We also need to compare covers with centers in the closure. If a nonempty set $K$ has a finite $\vep_g$-cover with centers in $K$, use doubling to cover each of these balls by finitely many balls of radius $\delta_g/2$. From each smaller ball meeting $K$, choose one point of $K$ as a new center. The triangle inequality shows that the chosen points form a finite $\delta_g$-cover of $K$. Conversely, every $\delta_g$-cover is a $\vep_g$-cover. Thus the same generatable samples have closures admitting the required finite internal covers at the two generator scales. The displayed covering-number inequalities then give
    \[
    \operatorname{C}_{\vep_a,\vep_g}(\Hcal)\le
    \operatorname{C}_{\delta_a,\delta_g}(\Hcal)\le
    C\operatorname{C}_{\vep_a,\vep_g}(\Hcal).
    \]
    This proves the uniform case. Applying the same comparison to every subclass in Theorem~\ref{thm:nonunifgen} proves the non-uniform case.
\end{proof}

It is well known that Euclidean space $\mathbb{R}^d$ is a doubling metric space. As a result, we have the following corollary:
\begin{corollary}
In any finite-dimensional normed vector space $\Xcal$, if $\Hcal\subseteq \{0,1\}^\Xcal$ satisfies the UUS property, then for any $\vep_a,\vep_g,\delta_a,\delta_g>0$, $\Hcal$ is $(\vep_a,\vep_g)$-uniformly (resp.\ non-uniformly) generatable if and only if it is $(\delta_a,\delta_g)$-uniformly (resp.\ non-uniformly) generatable.
\end{corollary}

For arbitrary parameter pairs, choose a positive pair strictly smaller in each coordinate than both pairs and apply Theorem~\ref{thm:doubling-eps-uniformgen} twice; finite-dimensional normed spaces are doubling. We next study how generatability behaves under changes of the underlying metric. The proof is similar to that of Theorem~\ref{thm:doubling-eps-uniformgen} so we defer it to Appendix~\ref{sec:appendix-other}.

\begin{theorem}
    \label{thm:doubling-metric}
    Let $(\Xcal,\rho_1)$ be a doubling metric space, and let $\rho_2$ be a metric on $\Xcal$ bi-Lipschitz equivalent to $\rho_1$. Suppose $\Hcal\subseteq\{0,1\}^\Xcal$ is a hypothesis class. Then, it satisfies the UUS property with respect to $\rho_1$ if and only if it satisfies the UUS property with respect to $\rho_2$. If these equivalent UUS conditions hold, then $\Hcal$ is uniformly (resp.\ non-uniformly) generatable with respect to $\rho_1$ if and only if it is uniformly (resp.\ non-uniformly) generatable with respect to $\rho_2$. 
\end{theorem}

So far, we have established the desired stability properties of uniform and non-uniform generation in doubling metric spaces. What about generation in the limit? Unfortunately, pathological phenomena can occur: generatability can change abruptly when changing the scale parameters $\vep_a$ or $\vep_g$. We present such an example in Example~\ref{eg:geninlim-doubling} here and leave the construction and proof to Appendix~\ref{sec:appendix-proof-diffeps}. 

\begin{example}
    \label{eg:geninlim-doubling}
    On the real line $\mathbb R$, there is a hypothesis class $\Hcal$ that satisfies the UUS property, and is $(\vep,1)$-generatable in the limit for any $\vep\in(0,1)$, but not $(1,1)$-generatable in the limit.
\end{example}
These examples illustrate phenomena that are specific to generation in metric spaces and are examined in more detail in the next section.


\subsection{Generatability is Scale-Sensitive}
Outside doubling spaces, these robustness properties can fail: generatability may depend sharply on $\varepsilon_a$, $\varepsilon_g$, and the underlying metric. The adversary reveals a sequence of positive examples from a hypothesis $h$, while the generator must eventually produce points in $\operatorname{supp}(h)$ at distance strictly greater than $\vep_g$ from every previously observed point. For uniform and non-uniform generation, the guarantee is triggered when the observed sample's $\vep_a$-covering number reaches a prescribed threshold. Increasing $\vep_a$ can only decrease this covering number, making the guarantee easier to satisfy. Decreasing $\vep_g$ also makes generation easier, since the generator is required to produce less novel instances.

In this section, we establish monotonicity results with respect to both $\vep_a$ and $\vep_g$. For \emph{uniform} and \emph{non-uniform} generation, the results align with this intuition: if a hypothesis class $\Hcal$ is $(\vep_a,\vep_g)$-uniformly (resp.\ non-uniformly) generatable, then it remains so when $\vep_a$ is increased or $\vep_g$ is decreased. We use Example~\ref{eg:l2-diffeps-case2} to illustrate this behavior in the $\ell^2$ space , where increasing $\vep_a$ or decreasing $\vep_g$ makes generation easier. Notably, this example also shows a separation between finite-dimensional normed spaces, and the infinite-dimensional space $\ell^2$. Depending on the choice of scale parameters $\vep_a$ and $\vep_g$, the same hypothesis class can exhibit all possible regimes of generatability, ranging from not generatable in the limit to uniformly generatable. This phenomenon cannot occur in a doubling space.

The situation for \emph{generation in the limit} is more subtle, as already illustrated by Example~\ref{eg:geninlim-doubling}.  When considering generation in the limit, monotonicity holds only when $\vep_a$ is decreased. In contrast, increasing $\vep_a$ can cause generatability to fail: a hypothesis class may be $(\vep_a,\vep_g)$-generatable in the limit for small $\vep_a$, yet lose this property once $\vep_a$ exceeds a threshold. Example~\ref{eg:l2-diffeps-case1} demonstrates that, in $\ell^2$, generatability in the limit can disappear when $\vep_a$ and $\vep_g$ exceed such thresholds. Moreover, in contrast to Example~\ref{eg:geninlim-doubling} in the doubling-space setting, Example~\ref{eg:l2-diffeps-case1} is more general: the threshold values of $\vep_a$ and $\vep_g$ at which generatability fails can be chosen arbitrarily.

\begin{figure}[t]
\centering

\begin{subfigure}[t]{0.45\linewidth}
    \centering
    \includegraphics[width=\linewidth]{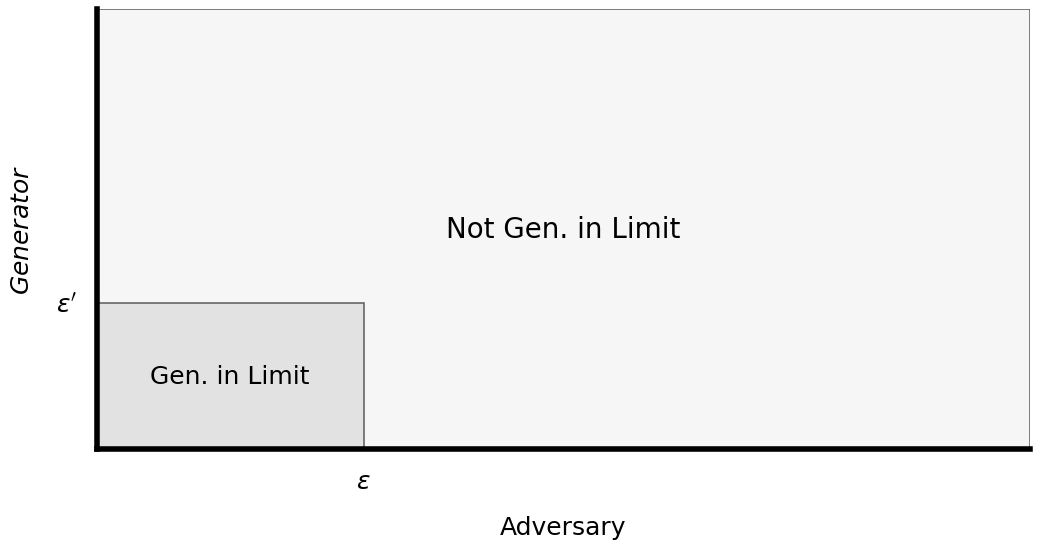}
    \caption{Illustration of Example~\ref{eg:l2-diffeps-case1}, showing generatability in the limit can disappear when $\vep_a$ and $\vep_g$ are larger than a threshold.}
    \label{fig:eg1}
\end{subfigure}
\hfill
\begin{subfigure}[t]{0.45\linewidth}
    \centering
    \includegraphics[width=\linewidth]{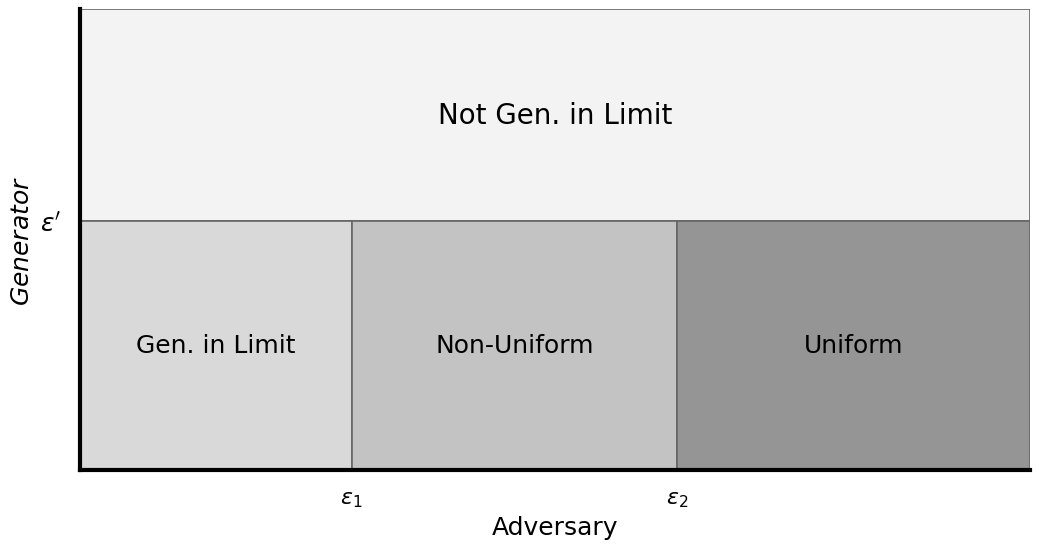}
    \caption{Illustration of Example~\ref{eg:l2-diffeps-case2}, showing increasing $\vep_a$ or decreasing $\vep_g$ makes generation easier.}
    \label{fig:eg2}
\end{subfigure}

\caption{Illustration of two generation regimes in a metric space. The horizontal axis represents the adversary parameter, and the vertical axis represents the generator parameter.}
\label{fig:main}
\end{figure}

\begin{theorem}\label{thm:gendiffeps}
    Let $\Hcal \subseteq \{0, 1\}^{\Xcal}$ be any hypothesis class satisfying the $r$-$\operatorname{UUS}$ property for some $r>0$ and let $0<\vep_a,\vep_g \le r$. If $\Hcal$ is $(\vep_a,\vep_g)$-generatable in the limit, then for any positive $\delta_a\le\vep_a$, $\delta_g\le\vep_g$, $\Hcal$ is $(\delta_a,\delta_g)$-generatable in the limit. 
\end{theorem}

\begin{theorem}\label{thm:unifdiffeps}
    Let $\Hcal \subseteq \{0, 1\}^{\Xcal}$  be any hypothesis class satisfying the $r$-$\operatorname{UUS}$ property for some $r>0$ and let $0<\vep_a,\vep_g \le r$. If $\Hcal$ is $(\vep_a,\vep_g)$-uniformly (or non-uniformly) generatable, then for any $\vep_a\le\delta_a\le r$ and  $0<\delta_g\le\vep_g$, $\Hcal$ is $(\delta_a,\delta_g)$-uniformly (or non-uniformly) generatable. 
\end{theorem}

Theorems~\ref{thm:gendiffeps} and~\ref{thm:unifdiffeps} establish monotonicity with respect to both $\vep_a$ and $\vep_g$ for all three notions of generation, which we prove in Appendix~\ref{sec:appendix-other}.

\begin{example}\label{eg:l2-diffeps-case1}
In $\ell^2$-space, for any $r>0$ and any $0<\vep_a,\vep_g\le r$, there exists a hypothesis class $\Hcal\subseteq\{0,1\}^{\ell^2}$ satisfying the $r$-$\operatorname{UUS}$ property such that, for every $0<\gamma_a,\gamma_g\le r$, it is $(\gamma_a,\gamma_g)$-generatable in the limit if and only if $\gamma_a<\vep_a$ and $\gamma_g<\vep_g$. 
    
\end{example}
Figure~\ref{fig:eg1} gives an illustration of this example. We now present the construction and the underlying intuition here, and defer the full proof to Appendix~\ref{sec:appendix-proof-diffeps}.

Recall that $\ell^2$ is the space that contains all sequences $(x_i)_{i=1}^\infty$ such that $\sum_i x_i^2<\infty$. Let $e_k \in \ell^2$ denote the $k$-th standard basis vector. Define
\[
u_k := 4r\cdot e_k, \qquad
a_{k} := \vep_a\cdot e_k, \qquad
g_k := \vep_g\cdot e_k .
\]
    Let $\mathcal I$ be the set of all infinite subsets of $\naturals_+$, i.e., $\mathcal I=\{I|I\in 2^{\naturals_+}, |I|=\infty\}$. For $I\in \mathcal I$, let 
    \[
    U_I=\{u_k|k\in I\},\quad O_I=\{a_{2k}|k\in I\}\cup\{g_{2k^n+1}|k\in I,n\in \naturals\},\quad D_I=\{g_{2k+1}|k\in I\}
    \]
    and
    construct the hypothesis class $\Hcal$ such that
    \[
    \Hcal=\{h_{I_1,I_2,I_3}|I_1,I_2, I_3\in\mathcal I,\quad \operatorname{supp}(h_{I_1,I_2,I_3})=\{
    \mathbf 0\}\cup U_{I_1}\cup O_{I_2}\cup D_{I_3}\}
    \]
    Here, $u_k$ and $U_I$ are constructed to ensure that $\Hcal$ satisfies the $r$-UUS property. As will be seen in the proof, the points $a_k$ are constructed for the adversary, while the points $g_k$ are constructed for the generator. In fact, the $g_k$ are the only instances with an explicit structure, which enables the generator to repeatedly produce them. The set $O_I$ is the \emph{optimal set}: when $\gamma_g<\vep_g$, revealing $a_{2k}\in O_I$ with $k\ge2$ identifies an infinite tower $\{g_{2k^n+1}:n\ge1\}$ from which novel outputs can be chosen. In contrast, $D_I$ is the \emph{disturbing set}: if the adversary continues to reveal $g_k \in D_I$, then the generator fails to generate in the limit.
     
    Intuitively, when $\gamma_a < \vep_a$ and $\gamma_g < \vep_g$, since the adversary must choose a sequence $\{x_i\}_{i=1}^\infty$ such that $\operatorname{supp}(h) \subseteq B(\{x_i\}_{i=1}^\infty,\gamma_a)$, there must exist some time $t$ at which $x_t = a_{2k}$ for some $k\ge2$. From that point on, the generator can choose a point $g_{2k^n+1}$ whose coordinate has not appeared in the finite observed sample; this point belongs to the target support and is more than $\gamma_g$ from every observation. When $\gamma_g \ge \vep_g$, however, as long as the adversary reveals $\mathbf{0}$, the generator cannot generate any $g_k$, since each $g_k$ lies within the $\gamma_g$-ball centered at $\mathbf{0}$. Finally, when $\gamma_a \ge \vep_a$, the adversary may choose to reveal only instances of the form $\mathbf{0}$, $g_k$, and $u_k$. The generator’s task is to uncover the structure of $O_I$ and generate from it. However, due to the interference caused by $D_I$, the generator cannot distinguish whether the observed instances originate from $O_I$ or from $D_I$, rendering generation in the limit impossible.

\begin{example}\label{eg:l2-diffeps-case2}
In $\ell^2$-space, for any $r>0$ and any
$0<\vep_{a1}\le \vep_{a2}\le r$ and $0<\vep_g\le r$,
there exists a hypothesis class $\Hcal \subseteq \{0,1\}^{\ell^2}$
satisfying the $r$-$\operatorname{UUS}$ property with the following
generation behavior. For $\gamma_g\in(0,\vep_g)$,
\[
\begin{array}{ll}
\gamma_a\in[\vep_{a2},r],
& \text{then $(\gamma_a,\gamma_g)$-uniformly generatable},\\[1mm]
\gamma_a\in[\vep_{a1},\vep_{a2}),
& \text{then $(\gamma_a,\gamma_g)$-non-uniformly but not uniformly generatable},\\[1mm]
\gamma_a\in(0,\vep_{a1}),
&  \text{then $(\gamma_a,\gamma_g)$-generatable in the limit but not non-uniformly generatable}.
\end{array}
\]
Moreover, for $\gamma_g\in[\vep_g,r]$ and any $\gamma_a\in(0,r]$,
$\Hcal$ is not $(\gamma_a,\gamma_g)$-generatable in the limit.
\end{example}

Figure~\ref{fig:eg2} illustrates this example. We give the construction and intuition here and prove the claims in Appendix~\ref{sec:appendix-proof-diffeps}. Let $K=\{2,3,\dots\}$. For $k\in K$, define points on four disjoint coordinate sets:
\[
\begin{aligned}
u_k&:=2r\,e_{4k}, & a_{k,1}&:=\vep_{a1}e_{4k+1},\\
a_{k,2}&:=\vep_{a2}e_{4k+2}, & g_k&:=\vep_g e_{4k+3}.
\end{aligned}
\]
Let $\mathcal I$ and $\mathcal I_{\mathrm{fin}}$ be the infinite and finite subsets of $K$, respectively. Set
\[
\begin{aligned}
O_I&=\{u_k:k\in I\}\cup\{g_{k^n}:k\in I,\ n\ge1\},\\
D_{I,J}&=\{a_{k,1}:k\in I\}\cup\{a_{k,2}:k\in J\},
\end{aligned}
\]
and define $\Hcal=\{h_{I_1,I_2,J}:I_1,I_2\in\mathcal I,\ J\in\mathcal I_{\mathrm{fin}}\}$ by
\[
\operatorname{supp}(h_{I_1,I_2,J})
=\{0\}\cup O_{I_1}\cup D_{I_2,J}.
\]
The $u$-points ensure $r$-UUS. Observing either $u_k$ or $g_k$ identifies an infinite tower $\{g_{k^n}:n\ge1\}$ in the target support. When $\gamma_g<\vep_g$, a tower point on an unobserved coordinate is always novel.

If $\gamma_a\ge\vep_{a2}$, the origin and all disturbing points in $D_{I_2,J}$ have covering number at most $1$, so a sample covering number of at least $2$ forces a tower observation. If $\vep_{a1}\le\gamma_a<\vep_{a2}$, these points have covering number at most $1+|J|$, giving a target-dependent threshold; arbitrarily large finite sets of second-family points rule out a uniform threshold. Below $\vep_{a1}$, arbitrary infinite subsets of first-family points obstruct non-uniform generation. Nevertheless, every valid covering sequence must reveal a $u$-point, so generation in the limit is possible whenever $\gamma_g<\vep_g$. Finally, when $\gamma_g\ge\vep_g$, observing $0$ blocks all $g$-points; a diagonal choice of the remaining support defeats generation in the limit.

\subsection{Generatability is Metric-Dependent}

In Section~\ref{sec:doubling}, we showed that uniform generatability is invariant under changes to equivalent metrics. In contrast, in general metric spaces, generatability can depend on the underlying metric: even equivalent metrics on the same instance space may induce different generatability properties.

We begin by establishing a stability result. Theorem~\ref{thm:diffmetricgen} shows that if one metric is bounded by another, then both the UUS property and generatability in the limit transfer in a controlled manner, with the relevant parameters scaling linearly with the bounding constant. In this sense, generatability remains stable under sufficiently mild distortions of the metric. We present the proof in Appendix~\ref{sec:appendix-other}.

\begin{theorem}
\label{thm:diffmetricgen}
    Let $\Xcal$ be a metric space with metric $\rho_1$ and $\rho_2$. Assume there exists $M>0$ such that for any $x, y\in \Xcal$, $ \rho_2(x,y)\le M\rho_1(x,y)$. Let $\Hcal\subseteq \{0,1\}^\Xcal$ be any class satisfying the $r$-UUS property for an $r>0$ with respect to $\rho_2$, and $0<\vep_a,\vep_g\le r$. Then $\Hcal$ satisfies the $r/M$-UUS property with respect to $\rho_1$. Moreover, if $\Hcal$ is $(\vep_a,\vep_g)$-generatable in the limit with respect to $\rho_2$, then $\Hcal$ is also $(\vep_a/M,\vep_g/M)$-generatable in the limit with respect to $\rho_1$.
\end{theorem}

Our next example highlights a sharp contrast. Even when two metrics are uniformly equivalent, generatability need not be preserved. In particular, we construct an example in $\ell^2$ showing that a hypothesis class can be uniformly generatable under the standard metric, yet fail to be non-uniformly generatable under an equivalent metric. We state the example and leave the full proof in Appendix~\ref{sec:appendix-proof-diffeps}.

\begin{example}\label{eg:l2-diffmetric}
    In $\ell^2$-space, for any $r>0$, there exists a hypothesis class $\Hcal$ that satisfies the $r$-UUS property and a metric $\rho'$ that is equivalent with the ordinary metric $\rho$ such that $\Hcal$ is $(\vep,\vep)$-uniform generatable with respect to $\rho$ for every $\vep\in(0,r)$, but $\Hcal$ is not $(\vep,\vep)$-non-uniformly generatable with respect to $\rho'$ for any $\vep\in(0,r)$.
\end{example}

\section{Discussion}

We initiated a theoretical study of generation in metric spaces, extending prior frameworks that were restricted to countable domains. Our results show that introducing geometric structure leads to new scale-dependent and metric-dependent phenomena, while preserving several older characterizations in modified form. In particular, we identified doubling metric spaces as a broad regime in which generatability behaves in a stable and robust manner, and demonstrated that outside this regime, generatability can be highly sensitive to both the novelty parameters and the choice of metric.

Several directions remain open, including studying generation when outputs are considered equivalent under natural symmetries, identifying geometric conditions beyond doubling that ensure stable behavior, and strengthening our understanding of generation under additional structural assumptions.

\bibliographystyle{plainnat}   
\bibliography{yourbibfile}

\appendix


\section{Summary of Results for Countable Example Space Generation}
\label{sec:appendix_countable}

We recall the formal definitions of generation in countable example spaces and the main structural results established in prior work, which motivate and inform our development of generation in metric spaces.

Through this section, $\Xcal$ is a countable example space and $\Hcal \subseteq \{0,1\}^{\Xcal}$ is a binary hypothesis class.

\subsection{Definitions of Generation}
\begin{definition}[Uniformly Unbounded Support (UUS)\citep{li2024generationlenslearningtheory}]
\label{def:countable-uus}
    A hypothesis class $\Hcal\subseteq\{0,1\}^\Xcal$ satisfies the \emph{Uniformly Unbounded Support (UUS) property} if $|\operatorname{supp}(h)|=\infty$ for every $h\in\Hcal$.
\end{definition}

Let $\Xcal$ be an countable example space and let $\Hcal \subseteq \{0,1\}^{\Xcal}$ be a hypothesis class satisfying the UUS property.

\begin{definition}[Generatability in the Limit {\citep{li2024generationlenslearningtheory}}]\label{def:countable-limit-gen}
The hypothesis class $\Hcal\subseteq{\{0,1\}}^\Xcal$ is \emph{generatable in the limit} if there exists a generator $\Gcal$ such that for every $h \in \Hcal$ and every enumeration $\{x_1,x_2,\dots\}$ of $\operatorname{supp}(h)$, there exists $t^\star \in \naturals$ such that $\Gcal(x_{1:s}) \in \operatorname{supp}(h)\setminus \{x_1,\dots,x_s\}
\quad$ for all $s \ge t^\star.$
\end{definition}

\begin{definition}[Uniform Generatability {\citep{li2024generationlenslearningtheory}}]\label{def:countable-uniform-gen}
The hypothesis class $\Hcal\subseteq\{0,1\}^\Xcal$ is \emph{uniformly generatable} if there exists a generator $\Gcal$ and an integer $d^\star \in \naturals$ such that for every $h \in \Hcal$ and every sequence $\{x_1,x_2,\dots\} \subseteq \operatorname{supp}(h)$, if there exists $t\in\naturals$ such that $|\{x_1,\dots,x_t\}|=d^\star$
then $\Gcal(x_{1:s}) \in \operatorname{supp}(h)\setminus \{x_1,\dots,x_s\}
\quad \text{for all } s \ge t.$

\end{definition}

\begin{definition}[Non-uniform Generatability {\citep{li2024generationlenslearningtheory}}]\label{def:countable-nonuniform-gen}
The hypothesis class $\Hcal\subseteq{\{0,1\}^\Xcal}$ is \emph{non-uniformly generatable} if there exists a generator $\Gcal$ such that for every $h \in \Hcal$, there exists an integer $d_h^\star \in \naturals$ such that for every sequence $\{x_1,x_2,\dots\} \subseteq \operatorname{supp}(h)$, if there exists $t\in\naturals$ such that $|\{x_1,\dots,x_t\}|=d_h^\star$, then $ \Gcal(x_{1:s}) \in \operatorname{supp}(h)\setminus \{x_1,\dots,x_s\} \quad \text{for all } s \ge t.$
\end{definition}

\subsection{Closure Dimension and Characterizations}

\begin{definition}[Closure Dimension {\citep{li2024generationlenslearningtheory}}]\label{def:countable-closure-dim}
The \emph{closure dimension} of a hypothesis class $\Hcal\subseteq\{0,1\}^\Xcal$, denoted $\operatorname{C}(\Hcal)$, is the largest integer $d \in \naturals$ such that there exist distinct points $x_1,\dots,x_d \in \Xcal$ with
$
\langle x_1,\dots,x_d \rangle_{\Hcal} \neq \bot
\quad \text{and} \quad
|\langle x_1,\dots,x_{d} \rangle_{\Hcal}|<\infty.
$
If such $d$ exists for arbitrarily large values, we write $\operatorname{C}(\Hcal)=\infty$.
\end{definition}

\begin{theorem}[Uniform Generatability Characterization {\citep{li2024generationlenslearningtheory}}]\label{thm:countable-uniform-char}
A hypothesis class $\Hcal \subseteq \{0,1\}^{\Xcal}$ is uniformly generatable if and only if
$
\operatorname{C}(\Hcal) < \infty.
$
\end{theorem}

\begin{theorem}[Non-uniform Generatability Characterization {\citep{li2024generationlenslearningtheory}}]\label{thm:countable-nonuniform-char}
A hypothesis class $\Hcal \subseteq \{0,1\}^{\Xcal}$ is non-uniformly generatable if and only if there exists a non-decreasing sequence of subclasses
$
\Hcal_1 \subseteq \Hcal_2 \subseteq \cdots
\quad \text{such that} \quad
\Hcal = \bigcup_{i=1}^\infty \Hcal_i
$
and
$
\operatorname{C}(\Hcal_i) < \infty
\quad \text{for all } i \in \naturals.
$
\end{theorem}

\begin{theorem}[Sufficient Condition for Generatability in the Limit {\citep{li2024generationlenslearningtheory}}]\label{thm:countable-suff-limit}
A hypothesis class $\Hcal \subseteq \{0,1\}^{\Xcal}$ is generatable in the limit if there exists a finite sequence of subclasses
$
\Hcal_1,\dots,\Hcal_n \subseteq \Hcal
  \text{ such that }
\Hcal = \bigcup_{i=1}^n \Hcal_i
$
and
$
\operatorname{C}(\Hcal_i) < \infty
 \text{ for all } i \in [n].
$
\end{theorem}

\section{Proofs in Section~\ref{sec:chargen}}

\label{sec:appendix-proof-chargen}
To simplify the proof, we give a definition of generatable sequence, which is a sequence that belongs to some hypothesis class $h$.

\begin{definition}[Generatable Sequence]
Let $\Hcal \subseteq \{0,1\}^{\Xcal}$ be a hypothesis class satisfying the $r$-$\operatorname{UUS}$ property for some $r>0$.  
A (finite or infinite) sequence $\{x_i\}_{i\in I}\subseteq \Xcal$ is called \emph{generatable} if there exists $h\in\Hcal$ such that
\[
\{x_i\}_{i\in I}\subseteq \operatorname{supp}(h),
\]
where $I=\{1,\dots,t\}$ for some $t\in\mathbb{N}$ or $I=\mathbb{N}$.
\end{definition}

\subsection{Proofs for Uniform Generation}

\begin{lemma}[Necessity in Theorem \ref{thm:unifgen}] \label{lem:closnecc}
Let $\Hcal \subseteq \{0, 1\}^{\Xcal}$  be any hypothesis class satisfying the $r$-$\operatorname{UUS}$ property for some $r>0$ and and let
$0<\vep_a,\vep_g \le r$. If $\operatorname{C}_{\vep_a,\vep_g}(\Hcal) = \infty$, then $\Hcal$ is \emph{not} $(\vep_a,\vep_g)$-uniformly generatable. 
\end{lemma}

\begin{proof}
Let $\Gcal$ be any generator and fix any proposed uniform threshold
$d\in\naturals$. Since
$\operatorname{C}_{\vep_a,\vep_g}(\Hcal)=\infty$, there exists a
generatable sequence $z_1,\dots,z_t$ such that
\[
N(\vep_a;\{z_i\}_{i=1}^t,\rho)\ge d
\]
and, writing $A=\langle z_1,\dots,z_t\rangle_{\Hcal}$, there is a
finite set $F\subseteq A$ with $A\subseteq B(F,\vep_g)$.

Append the points of $F$ as $z_{t+1},\dots,z_p$. Since every point
of $F$ belongs to every hypothesis in $\Hcal(z_{1:t})$, we have
\[
\Hcal(z_{1:p})=\Hcal(z_{1:t}).
\]
Let $\hat{x}_p=\Gcal(z_{1:p})$. If $\hat{x}_p\in A$, then
$\hat{x}_p\in B(F,\vep_g)$, so it is not novel; in this case,
choose any $h^\star\in\Hcal(z_{1:t})$. Otherwise, by the definition
of the closure, choose $h^\star\in\Hcal(z_{1:t})$ such that
$\hat{x}_p\notin\operatorname{supp}(h^\star)$.

In either case, the entire prefix lies in
$\operatorname{supp}(h^\star)$ and
\[
\Gcal(z_{1:p})\notin\operatorname{supp}(h^\star)
\setminus\bigcup_{i=1}^p B(z_i,\vep_g).
\]
Extend this prefix by any infinite sequence of positive examples
from $h^\star$. The generator fails at round $p\ge t$, after the
covering-number threshold $d$ has been reached. Since $\Gcal$ and
$d$ were arbitrary, $\Hcal$ is not
$(\vep_a,\vep_g)$-uniformly generatable.
\end{proof}

\begin{lemma}[Sufficiency in Theorem \ref{thm:unifgen}] \label{lem:clossuff}
Let $\Hcal \subseteq \{0, 1\}^{\Xcal}$  be any hypothesis class satisfying the $r$-$\operatorname{UUS}$ property for some $r>0$ and and let
$0<\vep_a,\vep_g \le r$. When $\operatorname{C}_{\vep_a,\vep_g}(\Hcal) < \infty$, there exists a generator $\Gcal$, such that for every $h \in \Hcal$ and any sequence $(x_i)_{i \in \naturals}$ with $\{x_1, x_2, \dots \} \subseteq \operatorname{supp}(h)$, if there exists $t$ such that $N(\vep_a;\{x_i\}_{i=1}^t,\rho)\ge\operatorname{C}_{\vep_a,\vep_g}(\Hcal)+1$, then $\Gcal(x_{1:s}) \in  \operatorname{supp}(h) \setminus\bigcup_{i=1}^sB(x_i,\vep_g)$ for all $s \geq t$.
\end{lemma}

\begin{proof}
Suppose $\operatorname{C}_{\vep_a,\vep_g}(\Hcal)=d<\infty$.
For every generatable sample $x_1,\dots,x_s$ satisfying
\[
N(\vep_a;\{x_i\}_{i=1}^s,\rho)\ge d+1,
\]
its closure $A_s=\langle x_1,\dots,x_s\rangle_{\Hcal}$ has no
finite $\vep_g$-cover with centers in $A_s$. Since
$\{x_i\}_{i=1}^s\subseteq A_s$, we obtain
\[
A_s\setminus\bigcup_{i=1}^s B(x_i,\vep_g)\ne\emptyset;
\]
otherwise, the observed points themselves would form such a cover.

Define $\Gcal$ to choose any point in this difference whenever
the sample is generatable and its covering number is at least
$d+1$, and to output an arbitrary point of $\Xcal$ otherwise.

Fix a target $h\in\Hcal$ and a sequence of its positive examples.
If the covering number is at least $d+1$ at round $t$, then
monotonicity of $N(\vep_a;\cdot,\rho)$ under set inclusion ensures
that it remains at least $d+1$ at every round $s\ge t$.
Moreover, $A_s\subseteq\operatorname{supp}(h)$. Therefore,
\[
\Gcal(x_{1:s})\in\operatorname{supp}(h)
\setminus\bigcup_{i=1}^s B(x_i,\vep_g)
\qquad\text{for all }s\ge t.
\]
This proves the desired uniform guarantee.
\end{proof}

Composing Lemmas \ref{lem:closnecc} and \ref{lem:clossuff} gives a characterization of uniform generatability. Next we prove Corollary~\ref{cor:genfinite}.

\begin{proof}[Proof of Corollary \ref{cor:genfinite}]
    We only need to show that $\operatorname{C}_{\vep_a,\vep_g}(\Hcal)<\infty$. To do this, suppose $\Hcal=\{h_1,\dots,h_s\}$. For any $A\subset [s]$ such that $A\neq\emptyset$, define
    \[
    N_{A,\vep_a}=N(\vep_a;\bigcap_{i\in A}\operatorname{supp}(h_i),\rho), \quad N_{A,\vep_g}=N(\vep_g;\bigcap_{i\in A}\operatorname{supp}(h_i),\rho).
    \]
    Since $\vep_g\le\vep_a$, we have $N_{A,\vep_a}\le N_{A,\vep_g}$. Define
    \[
    M:=\max\Bigl(\{0\}\cup\{N_{A,\vep_a}:\emptyset\neq A\subseteq[s],\ N_{A,\vep_g}<\infty\}\Bigr)<\infty.
    \]
    This maximum is finite because there are only finitely many nonempty subsets of $[s]$. Consider any generatable sample $x_1,\dots,x_t$ whose closure has a finite $\vep_g$-cover with centers in the closure. Let $A=\{i\in[s]:h_i\in\Hcal(x_1,\dots,x_t)\}$. Then $A\neq\emptyset$, the closure equals $\bigcap_{i\in A}\operatorname{supp}(h_i)$, and $N_{A,\vep_g}<\infty$. Since every observed point belongs to this intersection,
    \[
    N(\vep_a;\{x_i\}_{i=1}^t,\rho)\le N_{A,\vep_a}\le M.
    \]
    Thus $\operatorname{C}_{\vep_a,\vep_g}(\Hcal)\le M<\infty$, as required.
\end{proof}

We also provide an example of a hypothesis class $\Hcal$ consisting of only two hypotheses that is not non-uniformly generatable in $\ell^2$.
\begin{example}
    \label{eg:finitenotgen}
    For any $r>0$ and $0<\vep_a<\vep_g\le r$, there exists a two-hypothesis class $\Hcal\subseteq \{0,1\}^{\ell^2}$ satisfying the $r$-$\operatorname{UUS}$ property that is not $(\vep_a,\vep_g)$-non-uniformly generatable.
\end{example}
\begin{proof}
    Let $e_k$ be the $k$-th standard basis vector of $\ell^2$, and choose $\alpha$ with $\vep_a<\alpha<\vep_g$. Define
    \[
    \begin{aligned}
    a_k&:=\alpha e_{3k}, & S&:=\{0\}\cup\{a_k:k\ge1\},\\
    T_1&:=\{2r e_{3k+1}:k\ge1\}, & T_2&:=\{2r e_{3k+2}:k\ge1\}.
    \end{aligned}
    \]
    Let $\Hcal=\{h_1,h_2\}$, where
    \[
    \operatorname{supp}(h_1)=S\cup T_1,\qquad \operatorname{supp}(h_2)=S\cup T_2.
    \]
The $r$-UUS property is clear from the definition. For every $t\ge1$, both hypotheses contain $a_1,\dots,a_t$, and their supports intersect exactly in $S$. Hence
    \[
    \langle a_1,\dots,a_t\rangle_{\Hcal}=S\subseteq B(0,\vep_g),\qquad 0\in S.
    \]
    This is a finite cover with its center in the closure. But our definition of $a_i$ gives
    \[
    N(\vep_a;\{a_1,\dots,a_t\},\rho)\longrightarrow\infty,\qquad t\longrightarrow\infty.
    \]
    These samples witness $\operatorname{C}_{\vep_a,\vep_g}(\Hcal)=\infty$, so Theorem~\ref{thm:unifgen} rules out uniform generation. If a non-uniform generator existed, the maximum of its two target-dependent thresholds would be a uniform threshold for the same generator, a contradiction.
\end{proof}
\subsection{Proofs for Non-Uniform Generation}

We now prove Theorem \ref{thm:nonunifgen} across two lemmas, starting with the necessity direction.

\begin{lemma}[Necessity in Theorem \ref{thm:nonunifgen}]
Let $\Hcal \subseteq \{0, 1\}^{\Xcal}$  be any hypothesis class satisfying the $r$-$\operatorname{UUS}$ property for some $r>0$ and and let
$0<\vep_a,\vep_g \le r$. If $\Hcal$ is $(\vep_a,\vep_g)$-non-uniformly generatable, then there exists a sequence of \emph{non-decreasing, $(\vep_a,\vep_g)$-uniformly generatable} classes $\Hcal_1 \subseteq \Hcal_2 \subseteq \cdots$ such that $\Hcal =\bigcup_{n\in\naturals} \Hcal_n$.
\end{lemma}

\begin{proof} 
Suppose $\Hcal$ is $(\vep_a,\vep_g)$-non-uniformly generatable and $\Gcal$ is a $(\vep_a,\vep_g)$-non-uniform generator for $\Hcal$. For every $h\in\Hcal$, let $d_h\in\naturals$ be the smallest natural number such that for any sequence $\{x_i\}_{i=1}^\infty\subset \operatorname{supp}(h)$, if there exists $t$ such that $N(\vep_a;\{x_i\}_{i=1}^t,\rho)\ge d_h$, then  $\Gcal(x_{1:s})\in\operatorname{supp}(h)\setminus\bigcup_{i=1}^sB(x_i,\vep_g)$ for all $s\ge t$. Let $\Hcal_n:=\{h\in\Hcal:d_h \leq n\}$ for all $n\in\naturals$. Then, by the definition of $\Gcal$, we know for every $n\in\naturals$, $\Gcal$ is a $(\vep_a,\vep_g)$-uniform generator for $\Hcal_n$, and therefore $\Hcal_n$ is $(\vep_a,\vep_g)$-uniformly generatable.  The proof is complete after noting that $\Hcal_1 \subseteq \Hcal_2 \subseteq \cdots$ and $\Hcal=\bigcup_{n\in\naturals}\Hcal_n$.
\end{proof}

\begin{lemma}[Sufficiency in Theorem \ref{thm:nonunifgen}]
Let $\Hcal \subseteq \{0, 1\}^{\Xcal}$  be any hypothesis class satisfying the $r$-$\operatorname{UUS}$ property for some $r>0$ and and let
$0<\vep_a,\vep_g \le r$. If there exists a sequence of \emph{non-decreasing, $(\vep_a,\vep_g)$-uniformly generatable} classes $\Hcal_1 \subseteq \Hcal_2 \subseteq \dots$ such that $\Hcal = \bigcup_{n \in \naturals} \Hcal_n$, then $\Hcal$ is $(\vep_a,\vep_g)$-non-uniformly generatable. 
\end{lemma}

\begin{proof}
Suppose $\Hcal \subseteq \{0, 1\}^{\Xcal}$ is a class satisfying the $r$-$\operatorname{UUS}$  property for which there exists a sequence of non-decreasing, $(\vep_a,\vep_g)$-uniformly generatable classes $\Hcal_1 \subseteq \Hcal_2 \subseteq \dots$ with  $\Hcal = \bigcup_{n = 1}^{\infty} \Hcal_n$. By definition of uniform generatability, for every $n \in \naturals$, there exists a uniform generator $\Gcal_n$ for $\Hcal_n$. Let $d_{\Gcal_n}$ denote the $(\vep_a,\vep_g)$-uniform generation sample complexity of $\Gcal_n$ with respect to $\Hcal_{n}$.

Consider the following generator $\Gcal$. On any sample inconsistent with every hypothesis in $\Hcal$, let $\Gcal$ output an arbitrary point of $\Xcal$. Otherwise, fix $t \in \naturals$ and a sample $x_1,\dots,x_t$ such that $|\Hcal(x_1, \dots, x_t)| \geq 1$. Let  $d_t := N(\vep_a;\{x_i\}_{i=1}^t,\rho)$. $\Gcal$ first computes $n_t = \max\Bigl(\{n \in [t] : d_{\Gcal_n} \leq d_t\}\cup\{0\}\Bigr).$ If $n_t = 0$,   $\Gcal$ plays any $\hat{x}_t \in \Xcal$. If $n_t > 0$, $\Gcal$ uses $\Gcal_{n_t}$ to generate new instances, which means $\Gcal(x_{1},\dots,x_t)=\Gcal_{n_t}(x_1,\dots,x_t).$

We now prove that such a $\Gcal$ is a $(\vep_a,\vep_g)$-non-uniform generator for $\Hcal$. To that end, let $h^{\star}$ be the hypothesis chosen by the adversary and suppose that $h^{\star}$ belongs to $\Hcal_{n^{\star}}$. Let $d^{\star} = \max\{d_{\Gcal_{n^{\star}}}, n^{\star}\}.$ We show that for any sequence $\{x_i\}_{i=1}^\infty \subseteq \operatorname{supp}(h^{\star})$, if there exists $t^{\star} \in \naturals$ such that $N(\vep_a;\{x_i\}_{i=1}^{t^\star},\rho)\ge d^{\star}$, then $\Gcal(x_{1:s}) \in \operatorname{supp}(h^{\star}) \setminus \bigcup_{i=1}^s B(x_i,\vep_g)$ for all $s \geq t^{\star}$. 

Fix any valid sequence $\{x_i\}_{i=1}^\infty \subseteq \operatorname{supp}(h^{\star})$, and suppose that $N(\vep_a;\{x_i\}_{i=1}^{t^\star},\rho) \ge d^{\star}$ for some $t^{\star} \in \mathbb{N}.$ Fix any $s \geq t^{\star}.$ By definition, $\Gcal$ first computes 
\[
n_s=\max\Bigl(\{n \in [s] : d_{\Gcal_n} \leq d_s\} \cup \{0\}\Bigr).\]
Note that $n_s\ge n^{\star}$ since $s\ge t^\star\ge d^\star\ge n^\star$ (the covering number of a sample is at most its length) and $d_s\ge d^\star\ge d_{\Gcal_{n^\star}}$ by monotonicity of covering numbers. Thus, $|\Hcal_{n_s}(x_{1:s})| \geq 1$ since $h^{\star} \in \Hcal_{n^\star}$. Accordingly, by construction of $\Gcal$, it uses $\Gcal_{n_s}$ to generate a new instance. The proof is complete by noting that $h^{\star} \in \Hcal_{n_s}$ and $d_s \geq d_{\Gcal_{n_s}}$ which guarantees that $\Gcal_{n_s}(x_1,\dots, x_s)\in \operatorname{supp}(h^{\star})\setminus\bigcup_{i=1}^s B(x_i,\vep_g)$.

\end{proof}

We now show that every countable hypothesis class are non-uniformly generatable with $\vep_a\ge\vep_g$.

\begin{proof} [Proof of Corollary \ref{cor:countnonunif}]
    Suppose $\Hcal$ is a countable hypothesis class satisfying the $r$-UUS property. If $\Hcal$ is finite, Corollary~\ref{cor:genfinite} gives uniform, and hence non-uniform, generatability. Otherwise, $\Hcal$ is countably infinite; fix an enumeration $h_1,h_2,\dots$ without repetition. Let $\Hcal_n=\{h_1, \dots, h_n\}$ for all $ n\in\naturals$. Then, $\Hcal_1, \Hcal_2, \dots$ is a non-decreasing sequence of classes such that $\Hcal = \bigcup_{n \in \naturals} \Hcal_n.$ Moreover, since for every $n \in \naturals$, we have that $|\Hcal_{n}| = n < \infty$, Corollary \ref{cor:genfinite} gives that $\Hcal_n$ is $(\vep_a,\vep_g)$-uniformly generatable, completing the proof. 
\end{proof}

\subsection{Proofs for Generation in the Limit }
\begin{proof}[Proof of Theorem~\ref{thm:geninlim}]
Fix $r>0$ and $0<\vep_a,\vep_g\le r$. Define
\[
A=\{3rn:n\ge1\},\qquad B=\{3rn+\vep_a:n\ge1\},
\]
and let $\mathcal S$ consist of all sets $S_s=\{s_n:n\ge1\}$ with
\[
s_n\in(3rn,3rn+\vep_a)\qquad(n\ge1).
\]
Let $\Hcal_1=\{h_s^A:S_s\in\mathcal S\}$ and $\Hcal_2=\{h_s^B:S_s\in\mathcal S\}$, where
\[
\operatorname{supp}(h_s^A)=A\cup S_s,\qquad
\operatorname{supp}(h_s^B)=B\cup S_s.
\]

Clearly both classes and their union satisfies the $r$-UUS property. To see both classes are uniformly generatable, for $\Hcal_1$, after any nonempty finite sample, consider the generator that outputs the first point of $A$ exceeding the largest observation by more than $\vep_g$. This output belongs to every target in $\Hcal_1$ and is novel. Use $B$ in the same way for $\Hcal_2$. Thus, both classes are uniformly generatable with threshold $1$.

Now fix any generator $\Gcal$ for the union. Choose $s_n$ recursively in $(3rn,3rn+\vep_a)$ so that
\[
s_n\notin\{\Gcal(s_1,\dots,s_t):1\le t<n\}.
\]
Only finitely many points are excluded at each step, so this choice is possible. Set $S=\{s_n:n\ge1\}$. The sequence $(s_n)_{n\ge1}$ is positive for both fixed targets $h_s^A$ and $h_s^B$. Moreover,
\[
|s_n-3rn|<\vep_a,\qquad |s_n-(3rn+\vep_a)|<\vep_a,
\]
so it $\vep_a$-covers both supports. In fact, distinct observations are more than $3r-\vep_a\ge2r$ apart, and hence
\[
N(\vep_a;\{s_1,\dots,s_t\},|\cdot|)=t.
\]

The sets $A$, $B$, and $S$ are pairwise disjoint, so the intersection of the two target supports is exactly $S$. At any round $t$, the output $\Gcal(s_1,\dots,s_t)$ cannot equal any future $s_n$, by the recursive choice. If it belongs to $S$, it is therefore already observed and not novel. Otherwise it is outside at least one of the two target supports. Thus at every round the generator fails for at least one target. One of these two fixed targets consequently has infinitely many failure rounds along this admissible sequence. Since $\Gcal$ was arbitrary, the union is not $(\vep_a,\vep_g)$-generatable in the limit.
\end{proof}

\section{Constructing Examples in Section~\ref{sec:diffeps}}

\label{sec:appendix-proof-diffeps}

\subsection{Useful Lemmas for Constructing Hypothesis Classes}

This subsection gives a threshold-transfer lemma for the non-uniform lower bounds, a covering-number estimate for orthogonal points, and a criterion for failure of generation in the limit.

\begin{lemma}[Transfer of non-uniform thresholds]\label{lem:notnonunifgen}
Let $\Hcal$ satisfy $r$-$\operatorname{UUS}$ and let $0<\vep_a,\vep_g\le r$. Let $\mathcal S$ be a family of nonempty subsets of $\Xcal$, and let $T_1,\dots,T_q\subseteq\Xcal$ satisfy $\bigcap_{i=1}^qT_i=\emptyset$. Suppose that, for every $S\in\mathcal S$ and $i\in[q]$, $\Hcal$ contains a hypothesis $h_{S,i}$ with support $S\cup T_i$. If $\Gcal$ is an $(\vep_a,\vep_g)$-non-uniform generator for $\Hcal$, then for every $S\in\mathcal S$ there is a finite threshold $d_S$ such that, for every sequence $(x_j)_{j\ge1}$ in $S$, whenever
\[
N(\vep_a;\{x_1,\dots,x_t\},\rho)\ge d_S,
\]
we have
\[
\Gcal(x_{1:s})\in S\setminus B(\{x_1,\dots,x_s\},\vep_g)
\qquad\text{for every }s\ge t.
\]
\end{lemma}
Note that no UUS or coverage assumption is imposed on S. This Lemma will be useful in later constructions of examples
\begin{proof}
Let $d_{S,i}$ be a threshold for $\Gcal$ on $h_{S,i}$ and set $d_S=\max_{i\in[q]}d_{S,i}$. A sequence in $S$ is positive for every $h_{S,i}$. Once its covering number reaches $d_S$, the guarantee for each of these finitely many targets applies at every subsequent round. Intersecting their required output sets gives the conclusion, since
\[
\bigcap_{i=1}^q(S\cup T_i)=S\cup\bigcap_{i=1}^qT_i=S.
\]
\end{proof}

We will also use the following elementary estimate. For mutually orthogonal points $v_1,\dots,v_t$ of norm $R$ in a Hilbert space and $0<\eta<R$,
\begin{equation}\label{eq:orthogonal-cover-bound}
N(\eta;\{v_1,\dots,v_t\},\rho)
\ge t\left(1-\frac{\eta^2}{R^2}\right).
\end{equation}
Indeed, if a radius-$\eta$ ball with arbitrary center $z$ contains $q$ of these points, their mean $\bar v$ satisfies
\[
\eta^2\ge \frac1q\sum_{j=1}^q\|v_{i_j}-z\|^2
=R^2(1-1/q)+\|\bar v-z\|^2.
\]
Thus each ball contains at most $R^2/(R^2-\eta^2)$ points, proving the estimate. It applies both to the ordinary and to the weighted Hilbert norms used below.

The next lemma gives a way to defeat generation in the limit. The idea is to make the generator repeatedly commit to a target
that is consistent with the observations so far, and then continue
the sequence in a way that defeats that commitment. Thinking of the example spaces having infinitely many rows, and each hypothesis contains an entire row and some points from other rows. At stage $n$, condition 2
provides a target that guide the generator to reveal the entire $n$-th row. If the generator is tricked to commit to the $n$-th row, we
stop revealing this row and move to the next one. We also reveal
one point of the zeroth row at each stage. Condition 1 guarantees
that all the observations accumulated this way belong to, and
cover, a single final target with the zeroth row. For this final target, condition 3 makes sure that the generator will make a mistake at each stage $n$, and therefore infinitely
many errors eventually.

\begin{lemma}\label{lem:notgeninlimit}
Let $\Hcal \subseteq \{0,1\}^{\Xcal}$ satisfy the $r$-UUS property for some $r>0$, and let
$0<\vep_a,\vep_g\le r$.
Suppose there exist $k\in \mathbb{N}\cup\{0\}$, points $y_1,\dots,y_k\in \Xcal$,
a sequence $(x_{0i})_{i\ge 1}\subseteq \Xcal$, and an array $(x_{ij})_{i\ge 1,\,j\ge 1}\subseteq \Xcal$
such that the following hold.

\begin{enumerate}
\item There exists a set $A_0\subseteq \Xcal$ such that for every infinite sequence of positive
integers $(j_i)_{i\ge 1}$, there exists $h\in\Hcal$ containing all the specified observations:
\[
\{y_\ell:1\le\ell\le k\}\cup\{x_{0i}:i\ge1\}
\cup\bigcup_{i\ge1}\{x_{i\ell}:1\le\ell\le j_i\}
\subseteq\operatorname{supp}(h),
\]
and satisfying
\[
\operatorname{supp}(h)\subseteq B(\{y_\ell\}_{\ell=1}^k,\vep_a)\ \cup\ B(\{x_{0i}\}_{i\ge 1},\vep_a)\ \cup\
B\Bigl(\bigcup_{i\ge 1}\{x_{i\ell}:\,1\le \ell\le j_i\},\,\vep_a\Bigr),
\]
and
\[
\operatorname{supp}(h)\subseteq B(\{y_\ell\}_{\ell=1}^k,\vep_g)\ \cup\ A_0\ \cup\
B\Bigl(\bigcup_{i\ge 1}\{x_{i\ell}:\,1\le \ell\le j_i\},\,\vep_g\Bigr).
\]

\item For every $n\ge 1$, there exists a set $A_n\subseteq \Xcal$ such that for every finite sequence
of positive integers $(j_i)_{i=0}^{n-1}$, there exists $h\in\Hcal$ containing all the specified observations:
\[
\{y_\ell:1\le\ell\le k\}\cup\{x_{ni}:i\ge1\}
\cup\bigcup_{i=0}^{n-1}\{x_{i\ell}:1\le\ell\le j_i\}
\subseteq\operatorname{supp}(h),
\]
and satisfying
\[
\operatorname{supp}(h)\subseteq B(\{y_\ell\}_{\ell=1}^k,\vep_a)\ \cup\ B(\{x_{ni}\}_{i\ge 1},\vep_a)\ \cup\
B\Bigl(\bigcup_{i=0}^{n-1}\{x_{i\ell}:\,1\le \ell\le j_i\},\,\vep_a\Bigr),
\]
and
\[
\operatorname{supp}(h)\subseteq B(\{y_\ell\}_{\ell=1}^k,\vep_g)\ \cup\ A_n\ \cup\
B\Bigl(\bigcup_{i=0}^{n-1}\{x_{i\ell}:\,1\le \ell\le j_i\},\,\vep_g\Bigr).
\]

\item For every $n\ge 1$,
\[
A_n\ \cap\ \Bigl(A_0 \ \cup\ B(\{x_{ij}\}_{i\ge n+1,\ j\ge 1},\vep_g)\Bigr)\ =\ \emptyset.
\]
\end{enumerate}
Then $\Hcal$ is not $(\vep_a,\vep_g)$-generatable in the limit.
\end{lemma}

\begin{proof}
       Let $\Gcal$ be a generator that ensures $\Hcal$ is generatable in the limit. We show that there exists a hypothesis $h \in \Hcal$ and a sequence $\{z_{i}\}_{i=1}^\infty$, both depending on $\Gcal$, such that
\[
\{z_i:i\ge1\}\subseteq\operatorname{supp}(h)\subseteq B(\{z_i:i\ge1\},\vep_a),
\]
which leads to a contradiction.

We construct the sequence $\{z_{i}\}_{i=1}^\infty$ inductively. Intuitively, we want the adversary first reveals the points $y_1,\dots,y_k$, and then at stage $m$ reveals $x_{0m}$ together with a prefix of the $m$-th row $\{x_{m1},x_{m2},\dots\}$ until the generator outputs a point in $A_m$. 

At each stage $m$, condition 2 ensures that the already played prefix, the next point $x_{0m}$, and the entire $m$-th row belong to a single witnessing hypothesis $h_m$. Continuing with that row therefore gives a positive sequence satisfying the required $\vep_a$-coverage condition. The second inclusion in condition 2 places every point of $\operatorname{supp}(h_m)\setminus A_m$ within $\vep_g$ of the fixed observations already played before continuing the row. Thus eventual correct generation along this continuation forces an output in $A_m$ after a finite, nonempty row prefix.

Let $n_0=k$ and $z_1=y_1,z_2=y_2,\dots,z_{n_0}=y_k$. From our second assumption, there exists $h_1\in\Hcal$ such that 
    \[\operatorname{supp}(h_1)\subseteq B\bigl(\{y_i\}_{i=1}^k,\vep_a\bigr)
        \cup
        B\bigl(\{x_{1i}\}_{i=1}^\infty,\vep_a\bigr)
        \cup
        B( x_{01},
        \vep_a).\]
    Since $\Gcal$ is generatable in the limit, there exists $j_1\in\naturals$ such that 
    \[\Gcal(z_1,\dots,z_{n_0},x_{01},x_{11},\dots,x_{1j_1})\in A_1.\]
    Let $n_1=n_0+1+j_1$ and $z_{n_0+1}=x_{01},z_{n_0+2}=x_{11},z_{n_0+3}=x_{12},\dots,z_{n_1}=x_{1j_1}$, so  
    \[\Gcal(z_1,\dots,z_{n_1})\in A_1.\]
    
    Similarly, there exists $h_2\in\Hcal$ such that 
    \[\operatorname{supp}(h_2)\subseteq B\bigl(\{y_i\}_{i=1}^k,\vep_a\bigr)
        \cup
        B\bigl(\{x_{2i}\}_{i=1}^\infty,\vep_a\bigr)
        \cup
        B\Bigl( \{x_{0i}\}_{i=1}^2,
        \vep_a\Bigr)\cup
        B\Bigl( \{x_{1i}\}_{i=1}^{j_1},
        \vep_a\Bigr).\]
    Since $\Gcal$ is generatable in the limit, there exists $j_2\in\naturals$ such that 
    \[\Gcal(z_1,\dots,z_{n_1},x_{02},x_{21},\dots,x_{2j_2})\in A_2.\]
    Let $n_2=n_1+1+j_2$ and $z_{n_1+1}=x_{02},z_{n_1+2}=x_{21},z_{n_1+3}=x_{22},\dots,z_{n_2}=x_{2j_2}$, so  
    \[\Gcal(z_1,\dots,z_{n_2})\in A_2.\]

    Now, for some $m\ge 1$, assume we have defined $h_1,\dots,h_{m-1}, j_1,\dots,j_{m-1}$ and $ n_1,\dots, n_{m-1}$. Inductively, there exists $h_{m}\in\Hcal$ such that 
    \[\operatorname{supp}(h_m)\subseteq B\bigl(\{y_i\}_{i=1}^k,\vep_a\bigr)
        \cup
        B\bigl(\{x_{mi}\}_{i=1}^\infty,\vep_a\bigr)
        \cup
        B\Bigl( \{x_{0i}\}_{i=1}^m,
        \vep_a\Bigr)\cup
        B\Bigl( \bigcup_{i=1}^{m-1} \{x_{i\ell}\}_{\ell=1}^{j_i},
        \vep_a\Bigr).\]
    Since $\Gcal$ is generatable in the limit, there exists $j_m\in\naturals$ such that 
    \[\Gcal(z_1,\dots,z_{n_{m-1}},x_{0m},x_{m1},\dots,x_{mj_m})\in A_m.\]
    Let $n_m=n_{m-1}+1+j_{m}$ and $z_{n_{m-1}+1}=x_{0m},z_{n_{m-1}+2}=x_{m1},z_{n_{m-1}+3}=x_{m2},\dots,z_{n_{m}}=x_{mj_m}$, so  
    \[\Gcal(z_1,\dots,z_{n_m})\in A_m.\]

    Finally, condition 1 gives $h_0\in\Hcal$ containing the entire constructed sequence and such that 
    \[
    \operatorname{supp}(h_0)\subseteq B\bigl(\{y_i\}_{i=1}^k,\vep_a\bigr)
        \cup
        B\bigl(\{x_{0i}\}_{i=1}^\infty,\vep_a\bigr)
        \cup
        B\Bigl( \bigcup_{i=1}^{\infty} \{x_{i\ell}\}_{\ell=1}^{j_i},
        \vep_a\Bigr),
    \]
    which means 
    \[
    \operatorname{supp}(h_0)\subseteq B\bigl(\{z_i\}_{i=1}^\infty,\vep_a\bigr).
    \]
    However, from our construction and first assumption, for any $m\ge 1$, 
    \[\operatorname{supp}(h_0)\setminus B(\{z_i\}_{i=1}^{n_m},\vep_g)\subseteq \Bigl(A_0\cup B\Bigl(\{x_{ij}\}_{i\ge m+1,j\ge 1},\vep_g\Bigr)\Bigr )\]
    but from our third assumption,
    \[ A_m\cap \Bigl(A_0\cup B\Bigl(\{x_{ij}\}_{i\ge m+1
    ,j\ge 1},\vep_g\Bigr)\Bigr )=\emptyset,\]
    thus when observing $\{z_i\}_{i=1}^\infty$, $\Gcal$ will make mistakes infinite times. Thus $\Hcal$ is not generatable in the limit, which completes the proof.
\end{proof}

\subsection{Proof of Example~\ref{eg:geninlim-doubling}}
Let $\rho : \mathbb{R} \times \mathbb{R} \to \mathbb{R}_+$ be the metric defined by
\[
\rho(x,y) = |x-y| , \qquad x,y \in \mathbb{R}.
\]
Let $P$ denote the set of all positive odd prime integers. For each $p \in P$, define
\[
A_p = \{ p^n : n \in \mathbb{N} \} \,\cup\, \{ p^n - 1 : n \in \mathbb{N} \}.
\]
Let $2\mathbb{N} = \{2n : n \in \mathbb{N}\}$ denote the set of positive even integers. Define the hypothesis class
\[
\mathcal{H} =\Bigl\{ h \in \{0,1\}^{\mathbb{R}} \;\Big|\;
\operatorname{supp}(h) = A_p \cup B,\ \text{for some } p \in P \text{ and } B \subseteq 2\mathbb{N} \Bigr\}.
\]
Every support contains the unbounded set $A_p$.
For any $r>0$, a positive sequence covering the support
at radius $r$ must have unbounded range.
This range contains arbitrarily large finite sets with
pairwise distances greater than $2r$, so its prefix
$r$-covering numbers tend to infinity.
Hence $\Hcal$ satisfies $r$-UUS for every $r>0$.
We claim that $\Hcal$ is $(\vep,1)$-generatable in the limit
for every $\vep\in(0,1)$, but is not
$(1,1)$-generatable in the limit.

We first show that it is $(\vep,1)$-generatable in the limit for any $\vep\in(0,1)$. For any $h\in\Hcal$, $\{x_i\}_{i=1}^\infty\subseteq \operatorname{supp}(h)$, if $\operatorname{supp}(h)\subseteq B(\{x_i\}_{i=1}^\infty,\vep)$, then every support point must occur among the observations: all support points and observations are integers, so an omitted support point would have distance at least $1>\vep$ from the observed set. Thus the sequence enumerates $\operatorname{supp}(h)$. As a result, there exists some $t\in\naturals$ such that $x_t=p^k$ for some $p\in P,k\in\naturals$. Define $\Gcal$ such that if no observed point is a power of an odd prime, the $\Gcal(x_{1:s})=x_1$, otherwise, once a $p^k$ is observed, let $\Gcal(x_{1:s})\in\{p^n|n\in\naturals\}\setminus B(\{x_{i}\}_{i=1}^s,1)$. Such an output always exists because the prime powers are unbounded and finitely many radius-$1$ balls form a bounded set. This guarantees generation in the limit. \

Next, we use Lemma~\ref{lem:notgeninlimit} to show that it is not $(1,1)$-generatable in the limit.

We show that it is not $(1,1/2)$ generatable, combine with Theorem~\ref{thm:gendiffeps}, we then know it is not $(1,1)$-generatable in the limit.
Let $\{p_i\}_{i \ge 0}$ denote the sequence of odd prime numbers listed in increasing order, where $p_i$ is the $(i+1)$-st odd prime; in particular,
\[
p_0 = 3,\quad p_1 = 5,\quad p_2 = 7,\ \ldots
\]
Let
\[
x_{ij}=p_{i}^j-1,\qquad i\ge 0,j\ge 1.
\]
Apply Lemma~\ref{lem:notgeninlimit} with no initial $y$-points ($k=0$) and any UUS radius greater than $1$. We verify its three conditions.
\begin{enumerate}

\item For any infinite sequence of positive integers
    $\{j_i\}_{i=1}^\infty$, let $h\in\Hcal$ be such that 
    \[
    \operatorname{supp}(h)=A_{p_0}\cup \left(\bigcup_{i=1}^\infty \{x_{i\ell}\}_{\ell=1}^{j_i}\right)
    \]
    This support contains the entire current row and all the indicated row prefixes, as required by the membership condition of the lemma. It also satisfies 
    \[\operatorname{supp}(h)
        \subseteq
        B\bigl(\{x_{0i}\}_{i=1}^\infty,1\bigr)
        \cup
        B\Bigl(\bigcup_{i=1}^\infty \{x_{i\ell}\}_{\ell=1}^{j_i},
        1\Bigr).
    \]
    and 
    \[\operatorname{supp}(h)\subseteq A_{p_0}
        \cup
        B\Bigl(\bigcup_{i=1}^\infty \{x_{i\ell}\}_{\ell=1}^{j_i},
        1/2\Bigr).
    \]

    \item For any $n>0$ and any finite sequence of positive integers
    $\{j_i\}_{i=0}^{n-1}$, let $h\in\Hcal$ be such that 
    \[
    \operatorname{supp}(h)=A_{p_n}\cup \left(\bigcup_{i=0}^{n-1} \{x_{i\ell}\}_{\ell=1}^{j_i}\right)
    \]
    This support contains the entire current row and all the indicated row prefixes, as required by the membership condition of the lemma. It also satisfies 
    \[
        \operatorname{supp}(h)
        \subseteq
        B\bigl(\{x_{ni}\}_{i=1}^\infty,1\bigr)
        \cup
        B\Bigl(\bigcup_{i=0}^{n-1} \{x_{i\ell}\}_{\ell=1}^{j_i},
        1\Bigr).
    \]
    and 
    \[
    \operatorname{supp}(h)\subseteq A_{p_n}
        \cup
        B\Bigl(\bigcup_{i=0}^{n-1} \{x_{i\ell}\}_{\ell=1}^{j_i},
        1/2\Bigr).
    \]
    \item For any $n\in\naturals$, $A_{p_n}\cap A_{p_0}=\emptyset$. Moreover, every point of $A_{p_n}$ is an integer distinct from every $x_{ij}=p_i^j-1$ with $i\ge n+1$ and $j\ge1$: odd powers cannot equal these even integers, and equality between the even points would contradict uniqueness of prime factorization. Thus the distance between any such pair is at least $1$, so
    \[
    A_{p_n}\cap B(\{x_{ij}\}_{i\ge n+1,j\ge1},1/2)=\emptyset.
    \]

\end{enumerate}
Thus all three conditions of Lemma~\ref{lem:notgeninlimit} holds, makes $\Hcal$ not $(1,1/2)$-generatable in the limit.

\subsection{Proof of Example~\ref{eg:l2-diffeps-case1}}

\begin{proof}

Recall that $\ell^2$ is the space that contains all sequences $(x_i)_{i=1}^\infty$ such that $\sum_i x_i^2<\infty$. Let $e_k \in \ell^2$ denote the $k$-th standard basis vector. Define
\[
u_k := 4r\cdot e_k, \qquad
a_{k} := \vep_a\cdot e_k, \qquad
g_k := \vep_g\cdot e_k .
\]
    Let $\mathcal I$ be the set of all infinite subsets of $\naturals$, i.e., $\mathcal I=\{I|I\in 2^{\naturals}, |I|=\infty\}$. For $I\in \mathcal I$, let 
    \[
    U_I=\{u_k|k\in I\},\quad O_I=\{a_{2k}|k\in I\}\cup\{g_{2k^n+1}|k\in I,n\in \naturals\},\quad D_I=\{g_{2k+1}|k\in I\}
    \]
    and
    construct the hypothesis class $\Hcal$ such that
    \[
    \Hcal=\{h_{I_1,I_2,I_3}|I_1,I_2, I_3\in\mathcal I,\quad \operatorname{supp}(h_{I_1,I_2,I_3})=\{
    \mathbf 0\}\cup U_{I_1}\cup O_{I_2}\cup D_{I_3}\}
    \]
    
    Here, $u_k$ and $U_I$ are constructed to ensure that $\Hcal$ satisfies the $r$-UUS property. As will be seen in the proof, the points $a_k$ are constructed for the adversary, while the points $g_k$ are constructed for the generator. In fact, the $g_k$ are the only instances with an explicit structure, which enables the generator to repeatedly produce them. The set $O_I$ is the \emph{optimal set}: when $\gamma_g<\vep_g$, revealing $a_{2k}\in O_I$ with $k\ge2$ identifies an infinite tower $\{g_{2k^n+1}:n\ge1\}$ from which novel outputs can be chosen. In contrast, $D_I$ is the \emph{disturbing set}: if the adversary continues to reveal $g_k \in D_I$, then the generator fails to generate in the limit.

Every $u$-point in the support has distance at least $3r$
from every other support point:
the $a$- and $g$-points have norm at most $r$,
whereas distinct $u$-points are at distance $4\sqrt{2}r$.
Thus every positive sequence covering the support at
radius $r$ must observe all points of the infinite set $U_{I_1}$.
Since each radius-$r$ ball contains at most one $u$-point,
the prefix $r$-covering numbers tend to infinity.
This verifies $r$-UUS.

    \paragraph{Case 1: $\gamma_a\in(0,\vep_a)$ and $\gamma_g\in(0,\vep_g)$.}
    Fix a target $h$ and a positive sequence with $\operatorname{supp}(h)\subseteq B(\{x_i:i\ge1\},\gamma_a)$. Each $a_{2k}$ in the support has distance at least $\vep_a$ from every other support point. Indeed, its distance from $0$ is $\vep_a$, other $a$-points are at distance $\sqrt{2}\vep_a$, all $g$-points have odd coordinates and are at distance $\sqrt{\vep_a^2+\vep_g^2}$, and a $u$-point on the same coordinate is at distance $4r-\vep_a\ge\vep_a$ (other $u$-points are farther away). Therefore every such $a_{2k}$ must itself be observed, even under the infimum-based coverage definition.

    Since $I_2$ is infinite, eventually an observation $a_{2K}$ with $K\ge2$ occurs. Before the first such observation, let the generator output an arbitrary point. Afterwards, fix that $K$. Its tower $\{g_{2K^n+1}:n\ge1\}$ is infinite and lies in the target support. At each round $s$, choose the least $n\ge1$ for which the coordinate $2K^n+1$ has not occurred among the nonzero observations $x_1,\dots,x_s$, and output $g_{2K^n+1}$. Such an $n$ exists because the observed sample is finite. The output has distance $\vep_g>\gamma_g$ from $0$ and, by orthogonality, distance at least $\vep_g>\gamma_g$ from every other observation. Thus the generator is eventually always correct and novel.

    We next prove the two negative regimes, using Lemma~\ref{lem:notgeninlimit} in Case 2 and a direct diagonal argument in Case 3.

    \paragraph{Case 2: $\gamma_a\in[\vep_a,r]$ and  $\gamma_g\in (0,\vep_g)$.} Let $\{p_i\}_{i\ge1}=\{3,5,7,\dots\}$ be the sequence of odd prime numbers. Using Lemma \ref{lem:notgeninlimit}, let $k=1$, $y_1=\mathbf 0$, for any $i\ge 1, j\ge 1$ let 
    \[
    x_{ij}=\begin{cases}
    u_{p_i^{(j+1)/2}}& j \text{ is odd,}\\
    g_{2\cdot p_i^{j/2}+1}& j \text{ is even.}
    \end{cases}
    \]
    Moreover, let 
    \[
    x_{0j}=\begin{cases}
    u_{2^{(j+1)/2}}& j \text{ is odd,}\\
    g_{2\cdot 2^{j/2}+1}& j \text{ is even.}
    \end{cases}
    \]
    We verify that the conditions of Lemma \ref{lem:notgeninlimit} holds.
    \begin{enumerate}
        \item For any infinite sequence of positive integers $\{j_i\}_{i=1}^\infty$, from our construction of $x_{ij}$, we can find $J_1,J_2\in 2^\naturals$ such that 
        \[
        \bigcup_{i=1}^\infty\{x_{i\ell}\}_{\ell=1}^{j_i}=\{u_k|k\in {J_1}\}\cup \{g_{2k+1}|k\in J_2\}.
        \]
        Then, if we let $I_2=\{2^n|n\in\naturals\}$, $I_1=I_{2}\cup J_1$, $I_3=I_2\cup J_2$, 
        since $\gamma_a \ge \vep_a$, $B(\mathbf 0,\gamma_a)$ can cover all the points of the form $a_k$. Thus, $h=h_{I_1,I_2,I_3}$ satisfies
        \[
        \operatorname{supp}(h)\subseteq B(\mathbf 0,\gamma_a)\cup B(\{x_{0i}\}_{i=1}^\infty,\gamma_a)\cup B(\bigcup_{i=1}^\infty\{x_{i\ell}\}_{\ell=1}^{j_i},\gamma_a)
        \]
        Moreover, if \[
        A_0
        =
        \bigcup_{k\in\naturals}
        \bigl\{
        u_{2^k},
        a_{2\cdot 2^k},
        g_{2\cdot 2^k+1}
        \bigr\},
        \]
        that is, $A_0$ contains the ``protected'' part for the $n=0$ hypothesis (a whole copy of a $u/a/g$-pattern indexed by powers of 2). We then know 
        \[
        \operatorname{supp}(h)\subseteq B(\mathbf 0,\gamma_g)\cup A_0\cup B(\bigcup_{i=1}^\infty\{x_{i\ell}\}_{\ell=1}^{j_i},\gamma_g).
        \]
        The support also contains $0$, the entire zeroth row, and every selected row prefix. Thus both the membership and covering requirements of the first condition hold.
        \item For $n\ge1$ and any prescribed prefixes of rows $0,\dots,n-1$, write their union as
        \[
        \bigcup_{i=0}^{n-1}\{x_{i\ell}:1\le\ell\le j_i\}
        =\{u_k:k\in J_1\}\cup\{g_{2k+1}:k\in J_2\}.
        \]
        Set $Q_n=\{p_n^k:k\ge1\}$ and choose
        \[
        I_1=Q_n\cup J_1,\qquad I_2=Q_n,\qquad I_3=Q_n\cup J_2.
        \]
        All three sets are infinite. The resulting support contains $0$, all prescribed prefixes, and the entire $n$-th row. Since powers of powers of $p_n$ remain powers of $p_n$, its remaining points are exactly the $a$-points in
        \[
        A_n=\bigcup_{k\ge1}\{u_{p_n^k},a_{2p_n^k},g_{2p_n^k+1}\}.
        \]
        These $a$-points lie in $B(0,\gamma_a)$, so the adversarial covering requirement holds. The whole support is contained in $\{0\}\cup A_n$ together with the prescribed prefixes, which gives the generator-scale covering requirement.

        \item Uniqueness of prime factorization gives $A_n\cap A_0=\emptyset$: within each point type the underlying prime powers differ; $a$-points have even coordinates whereas $g$-points have odd coordinates; and $u$-points have norm $4r$, different from either other type. Consider a point of $A_n$ and a point in a later row $i>n$, which contains only $u$- and $g$-points. Distinct $u$-points are at distance $4\sqrt{2}r$, and distinct $g$-points are at distance $\sqrt{2}\vep_g$. A $u$-point and a $g$-point, including when they share a coordinate, have distance at least $4r-\vep_g\ge3r$. An $a$-point in $A_n$ has an even coordinate, whereas both types in a later row have odd coordinates; its distances to them are $\sqrt{\vep_a^2+16r^2}$ and $\sqrt{\vep_a^2+\vep_g^2}$. All these distances are at least $\vep_g>\gamma_g$. This uniform bound also excludes membership in the infimum-based neighborhood of all future rows. Hence condition 3 holds.
    \end{enumerate}
    All three conditions of Lemma~\ref{lem:notgeninlimit} hold, so $\Hcal$ is not generatable in the limit when $\gamma_a\ge\vep_a$ and $0<\gamma_g<\vep_g$.

    \paragraph{Case 3: $\gamma_g\in[\vep_g,r]$.}
    Fix any $0<\gamma_a\le r$ and any generator $\Gcal$. Let
    \[
    D=\{g_{2k+1}:k\ge1\}.
    \]
    For every infinite set $P$ of odd primes, the choices
    \[
    I_1=\{2p:p\in P\},\qquad I_2=P,\qquad I_3=\mathbb N_+
    \]
    define a hypothesis $h_P\in\Hcal$ with
    \[
    \operatorname{supp}(h_P)=\{0\}\cup D\cup\{u_{2p},a_{2p}:p\in P\}.
    \]
    Indeed, all $g$-towers required by $I_2=P$ already lie in $D$.

    We choose $P$ and a sequence of positive observations inductively. First reveal $0$. Maintain finite sets of selected and forbidden odd primes, initially empty. At stage $m\ge1$, choose an odd prime $p_m$ in neither set, add it to the selected set, and reveal $g_{2m+1}$, $u_{2p_m}$, and $a_{2p_m}$. Inspect the output of $\Gcal$ at the end of this stage. If it is $u_{2q}$ or $a_{2q}$ for an odd prime $q$ not yet selected, add $q$ to the forbidden set. Only finitely many primes are unavailable at any finite stage, so the construction continues indefinitely.

    Let $P=\{p_m:m\ge1\}$. This set is infinite, and the resulting sequence enumerates exactly $\operatorname{supp}(h_P)$, hence is positive and satisfies the coverage requirement for every $\gamma_a>0$. At the end of each stage, an output in $\{0\}\cup D$ is not novel because all such points lie in $B(0,\gamma_g)$. An output from a pair belonging to an already selected prime has itself been observed. An output from a pair belonging to an unselected odd prime is excluded from the final support by the forbidden-set rule. Every other output lies outside $\operatorname{supp}(h_P)$. Thus $\Gcal$ fails at the end of every stage. Since $\Gcal$ was arbitrary, $\Hcal$ is not $(\gamma_a,\gamma_g)$-generatable in the limit throughout this regime, including $\gamma_g=\vep_g$.
\end{proof}

\subsection{Proof of Example~\ref{eg:l2-diffeps-case2}}

\begin{proof}
Use the construction in the main text: $K=\{2,3,\dots\}$,
\[
\begin{aligned}
u_k&=2r\,e_{4k}, & a_{k,1}&=\vep_{a1}e_{4k+1},\\
a_{k,2}&=\vep_{a2}e_{4k+2}, & g_k&=\vep_g e_{4k+3},
\end{aligned}
\]
and
\[
\begin{aligned}
O_I&=\{u_k:k\in I\}\cup\{g_{k^n}:k\in I,\ n\ge1\},\\
D_{I,J}&=\{a_{k,1}:k\in I\}\cup\{a_{k,2}:k\in J\},\\
\operatorname{supp}(h_{I_1,I_2,J})&=\{0\}\cup O_{I_1}\cup D_{I_2,J},
\end{aligned}
\]
where $I_1,I_2\subseteq K$ are infinite and $J\subseteq K$ is finite. The four coordinate sets are disjoint even when some scale parameters coincide. Every $u$-point has distance at least $2r$
from every other support point.
Hence every positive sequence covering the support at
radius $r$ must observe all its infinitely many $u$-points.
Their pairwise distances are $2\sqrt{2}r>2r$,
so each radius-$r$ ball contains at most one of them.
The sample's $r$-covering number is therefore at least
the number of observed $u$-points and tends to infinity.
This proves $r$-UUS.

For all the positive results below, use the following single generator. Until an observation of the form $u_k$ or $g_k$ occurs, output $0$. After the first such observation, retain its index $k\ge2$ and, at each round, output the first point of $\{g_{k^n}:n\ge1\}$ whose coordinate has not appeared among the nonzero observations. This defines the generator on positive samples; on other inputs it may output arbitrarily. The tower belongs to the target: if $g_k$ was observed, then $k=j^b$ for some $j\in I_1$ and $b\ge1$, and $k^n=j^{bn}$. The tower is infinite because $k\ge2$. Its selected point has distance $\vep_g$ from $0$ and, by orthogonality, at least $\vep_g$ from every observation on another coordinate. Thus every output after the first tower observation is correct and novel whenever $\gamma_g<\vep_g$.

\paragraph{Case 1: $\gamma_a\in[\vep_{a2},r]$ and $\gamma_g\in(0,\vep_g)$.}
The set $\{0\}\cup D_{I_2,J}$ lies in $B(0,\gamma_a)$. Therefore a positive sample with $\gamma_a$-covering number at least $2$ must contain a point of $O_{I_1}$. The generator just defined is consequently uniform with threshold $2$.

\paragraph{Case 2: $\gamma_a\in[\vep_{a1},\vep_{a2})$ and $\gamma_g\in(0,\vep_g)$.}
The origin and all first-family $a$-points lie in $B(0,\gamma_a)$, and the $|J|$ second-family points can each be covered by a ball centered at that point. Thus
\[
N(\gamma_a;\{0\}\cup D_{I_2,J},\rho)\le 1+|J|.
\]
A sample covering number of at least $|J|+2$ forces an observation from $O_{I_1}$. The same generator is therefore non-uniform.

To rule out uniform generation, let $F_t=\{a_{k,2}:2\le k\le t+1\}$. Its closure is exactly $\{0\}\cup F_t$. Indeed, choose $J=\{2,\dots,t+1\}$ to exclude every other second-family point; any first-family point can be excluded by choosing $I_2$ to omit its index; and the two choices $I_1=\{2^j:j\ge1\}$ and $I_1=\{3^j:j\ge1\}$ have disjoint $O$-sets. This finite closure has a finite internal $\gamma_g$-cover, whereas \eqref{eq:orthogonal-cover-bound} gives
\[
N(\gamma_a;F_t,\rho)
\ge t\left(1-\frac{\gamma_a^2}{\vep_{a2}^2}\right)\longrightarrow\infty.
\]
Hence $\operatorname{C}_{\gamma_a,\gamma_g}(\Hcal)=\infty$, and Theorem~\ref{thm:unifgen} rules out uniform generation. If $\vep_{a1}=\vep_{a2}$, this case is empty.

\paragraph{Case 3: $\gamma_a\in(0,\vep_{a1})$ and $\gamma_g\in(0,\vep_g)$.}
First, generation in the limit holds. Every $u_k$ has distance at least $2r$ from every other support point: its coordinate set is disjoint from the other point types, and different $u$-points are farther apart. Since $\gamma_a\le r<2r$, a positive sequence covering the support at radius $\gamma_a$ must observe each $u_k$ in the support. In particular it eventually identifies a tower, after which the generator is always correct and novel.

For the non-uniform lower bound, set
\[
S_I=\{0\}\cup\{a_{k,1}:k\in I\},
\qquad
T_1=O_{\{2^j:j\ge1\}},\quad T_2=O_{\{3^j:j\ge1\}}.
\]
Here $I$ ranges over infinite subsets of $K$. The sets $T_1,T_2$ are disjoint, and $S_I\cup T_i$ is a support in $\Hcal$ (take $J=\emptyset$). Suppose a non-uniform generator $\Gcal$ existed. Lemma~\ref{lem:notnonunifgen} would give it a finite threshold on every $S_I$.

We construct a contradiction to that conclusion. Reveal $0$, and maintain finite sets of selected and forbidden indices in $K$, initially empty. At stage $t$, choose $k_t$ in neither set, select it, and reveal $a_{k_t,1}$. If the resulting output is $a_{q,1}$ for an unselected $q\in K$, forbid $q$ for all subsequent stages. Let $I=\{k_t:t\ge1\}$. The sequence is positive for $S_I$. At the end of each stage, outputs equal to $0$ or to already selected points are not novel, outputs at unselected first-family indices are excluded from $S_I$, and every other output is outside $S_I$. Nevertheless, after $t$ stages its sample covering number is at least
\[
t\left(1-\frac{\gamma_a^2}{\vep_{a1}^2}\right)\longrightarrow\infty
\]
by \eqref{eq:orthogonal-cover-bound}. It therefore eventually exceeds the finite threshold supplied by the lemma, a contradiction. This argument does not assume any coverage or generator-scale UUS property for $S_I$.

\paragraph{Case 4: $\gamma_g\in[\vep_g,r]$.}
Fix any $0<\gamma_a\le r$ and any generator $\Gcal$. Restrict to targets with $I_1=I_2=P$, $J=\emptyset$, where $P$ is an infinite set of odd primes. Their supports are
\[
\{0\}\cup\{u_p,a_{p,1}:p\in P\}
\cup\{g_{p^n}:p\in P,\ n\ge1\}.
\]
Reveal $0$. At stage $t$, choose a previously unselected and unforbidden odd prime $p_t$, and reveal $u_{p_t}$ and $a_{p_t,1}$. Also reveal all not-yet-observed $g_{p_i^n}$ with $1\le i,n\le t$. At the end of the stage, if the generator outputs $u_q$ or $a_{q,1}$ for an unselected odd prime $q$, forbid $q$. Only finitely many primes have been selected or forbidden at any finite stage, so the construction continues.

For $P=\{p_t:t\ge1\}$, the resulting sequence enumerates exactly the target support and is valid for every $\gamma_a>0$. At every stage end, all $g$-points and $0$ are non-novel because they lie in $B(0,\gamma_g)$. Points in pairs for selected primes have already been observed, and predictions of pairs for unselected primes are excluded from the final support. All other outputs are outside the support. Thus $\Gcal$ fails infinitely often, proving the claimed impossibility, including the boundary $\gamma_g=\vep_g$.
\end{proof}

\subsection{Proof of Example~\ref{eg:l2-diffmetric}}

\begin{proof}
Fix $r>0$. Retain the weighted Hilbert metric
\[
\rho'(x,y)=\left(\sum_{j\ge1}\delta_j(x_j-y_j)^2\right)^{1/2},
\qquad
\delta_j=
\begin{cases}
1,&j\text{ odd},\\
1/4,&j\text{ even}.
\end{cases}
\]
Writing $\rho$ for the ordinary $\ell^2$ metric, we have $\rho/2\le\rho'\le\rho$, so the metrics are bi-Lipschitz equivalent. For $m,k\ge1$, define
\[
\alpha_m=2^{3-m}r,\qquad a_{m,k}=\alpha_m e_k.
\]
Let $\mathcal I$ be the infinite subsets of $\mathbb N_+=\{1,2,\dots\}$, and set
\[
A_{m,I}=\{a_{m,\,2^n(2k+1)}:k\in I,\ n\ge0\}.
\]
Define $\Hcal=\{h_{\mathbf I}:\mathbf I=(I_m)_{m\ge1}\in\mathcal I^{\mathbb N_+}\}$ by
\[
\operatorname{supp}(h_{\mathbf I})=\{0\}\cup\bigcup_{m\ge1}A_{m,I_m}.
\]
Distinct layers consist of points with different ordinary norms, so their point sets are disjoint. Within a layer, different odd indices generate disjoint doubling orbits.

Every support contains infinitely many first-layer
odd-coordinate points, each of norm $\alpha_1=4r$
under both metrics.
Each such point has distance at least $2r$ from every
other support point:
later-layer points have norm at most $2r$,
and distinct first-layer points occupy orthogonal coordinates.
Thus a positive sequence covering the support at radius $r$,
under either metric, must observe all these odd-coordinate points.
Their pairwise distances are $4\sqrt{2}r>2r$ under both metrics,
so every radius-$r$ ball contains at most one of them.
The prefix $r$-covering numbers therefore tend to infinity.
This proves $r$-UUS under both metrics.

\paragraph{Uniform generation under $\rho$.}
Fix $0<\vep<r$. If a positive sample has $\vep$-covering number at least $2$, it contains a point of ordinary norm greater than $\vep$: otherwise the sample would lie in $B_\rho(0,\vep)$. Write such an observation as $a_{m,k}$, where $\alpha_m>\vep$. Its full forward orbit
\[
\{a_{m,2^nk}:n\ge1\}
\]
belongs to the target support. Indeed, multiplying its coordinate by another power of $2$ preserves the odd part indexing that orbit.

Define the generator to output $0$ until the first observation of norm greater than $\vep$. Thereafter retain that observation $a_{m,k}$ and at each round output the first point in its forward orbit whose coordinate is absent from the nonzero observations. The orbit is infinite, so this is always possible. The output has distance $\alpha_m>\vep$ from $0$ and distance at least $\alpha_m>\vep$ from every observation on another coordinate. Thus the generator is correct and novel as soon as the sample covering number reaches $2$, and at every later round. Define its output arbitrarily on other inputs. This proves $(\vep,\vep)$-uniform generation with threshold $2$.

\paragraph{Failure of non-uniform generation under $\rho'$.}
Fix any $0<\vep<r$. Choose the unique $m\ge3$ such that
\[
\alpha_m/2\le\vep<\alpha_m.
\]
These intervals partition $(0,r)$, including their lower endpoints. In layer $m$, odd-coordinate points have weighted norm $\alpha_m>\vep$, while all even-coordinate descendants have weighted norm $\alpha_m/2\le\vep$.

For $I\in\mathcal I$, let $S_I=\{0\}\cup A_{m,I}$. Choose the two disjoint infinite index sets
\[
E_1=\{2j:j\ge1\},\qquad E_2=\{2j-1:j\ge1\},
\]
and define completions
\[
T_b=\bigcup_{\ell\ge1,\ \ell\ne m}A_{\ell,E_b},
\qquad b=1,2.
\]
The sets $T_1,T_2$ are disjoint: different layers have different norms, and within any layer the disjoint sets of odd parts give disjoint orbits. For every $I$, both $S_I\cup T_1$ and $S_I\cup T_2$ are supports in $\Hcal$.

Suppose that $\Gcal$ were an $(\vep,\vep)$-non-uniform generator under $\rho'$. Lemma~\ref{lem:notnonunifgen}, applied with this metric, would give a finite threshold for $\Gcal$ on each $S_I$. We contradict this using only the origin and odd-coordinate observations from layer $m$.

Reveal $0$. Maintain finite sets of selected and forbidden indices in $\mathbb N_+$, initially empty. At stage $t$, select an index $k_t$ in neither set and reveal $a_{m,2k_t+1}$. If the generator outputs an odd-coordinate point $a_{m,2q+1}$ with $q\ge1$ unselected, forbid $q$. Let $I=\{k_t:t\ge1\}$. This defines an infinite set and a sequence positive for $S_I$. At every stage end, an output at a selected odd-coordinate point has already been observed; an output at an unselected odd-coordinate point is excluded from $S_I$. Every even-coordinate point in $S_I$ lies in $B_{\rho'}(0,\vep)$ and is therefore non-novel, including when $\alpha_m/2=\vep$. The origin is also non-novel, and all other outputs are outside $S_I$.

Finally, the observed odd-coordinate points are mutually orthogonal with weighted norm $\alpha_m$. By \eqref{eq:orthogonal-cover-bound},
\[
N(\vep;\{0,a_{m,2k_1+1},\dots,a_{m,2k_t+1}\},\rho')
\ge t\left(1-\frac{\vep^2}{\alpha_m^2}\right)\longrightarrow\infty.
\]
Thus the sequence eventually crosses the finite threshold supplied by the transfer lemma while the generator continues to fail, a contradiction. Since $\vep\in(0,r)$ was arbitrary, the same class fails non-uniform generation under $\rho'$ at every scale required by the statement.
\end{proof}

\section{Other Technical Details}

\label{sec:appendix-other}
\subsection{Proofs Related to the UUS Property}
This subsection collects short proofs establishing a basic monotonicity result for the UUS property in general metric spaces, as well as a stronger scale-invariance result in doubling spaces. We first show that the UUS property is monotone in the scale parameter.

\begin{proof}[Proof of Lemma~\ref{lem:uusdiffeps}]
Fix $0<\delta<r$, $h\in\Hcal$, and a sequence
$(x_i)_{i\ge1}$ in $\operatorname{supp}(h)$ such that
\[
\operatorname{supp}(h)\subseteq B(\{x_i:i\ge1\},\delta).
\]
Since $\delta<r$, this sequence also covers the support
at radius $r$. Hence $r$-UUS and monotonicity of covering
numbers give
\[
N(\delta;\{x_1,\dots,x_t\},\rho)
\ge N(r;\{x_1,\dots,x_t\},\rho)
\longrightarrow\infty.
\]
The supports are nonempty by assumption, so $\Hcal$
satisfies $\delta$-UUS.
\end{proof}

The previous lemma shows monotonicity of the UUS property in the scale parameter. Under the additional assumption that $(\Xcal,\rho)$ is doubling, we now show that UUS is in fact scale-invariant.

\begin{proof}[Proof of Theorem~\ref{thm:doubling-UUS}]
Fix $h\in\Hcal$.
Its support is nonempty and unbounded.
Indeed, if it were bounded, doubling would give a finite
radius-$r$ cover of the support.
A countable dense sequence in the support would then cover
it at radius $r$ while having uniformly bounded prefix
$r$-covering numbers, contradicting $r$-UUS.

Now fix $r'>0$ and any positive sequence covering the support
at radius $r'$.
Its range must be unbounded, and therefore contains
arbitrarily large finite sets with pairwise distances
greater than $2r'$.
Each such set occurs in a finite prefix, and each
radius-$r'$ ball contains at most one of its points.
Consequently,
\[
N(r';\{x_1,\dots,x_t\},\rho)\longrightarrow\infty.
\]
Thus $\Hcal$ satisfies $r'$-UUS for every $r'>0$.
\end{proof}

\subsection{Recovering the Countable Framework}

This subsection shows that our metric-space formulation strictly generalizes the countable example space model studied in \cite{li2024generationlenslearningtheory}. In particular, when the instance space is equipped with the discrete metric at a sufficiently small scale, our notions of novelty, coverage, and generatability reduce exactly to their counterparts in the countable setting. Thus, all generatability notions considered in this paper are consistent with the classical framework.

\begin{proposition}\label{thm:discretecontinuous}
Let $\Xcal$ be a countable set equipped with the discrete metric
\[
\rho(x,y)=
\begin{cases}
1, & x\neq y,\\
0, & x=y.
\end{cases}
\]
Let $r<1$ and $0<\vep_a,\vep_g\le r$. Suppose $\Hcal\subseteq\{0,1\}^{\Xcal}$ satisfies the $r$-UUS property.
Then $\Hcal$ is $(\vep_a,\vep_g)$-generatable in the limit (respectively, $(\vep_a,\vep_g)$-non-uniformly generatable, $(\vep_a,\vep_g)$-uniformly generatable) if and only if $\Hcal$ is generatable in the limit (respectively, non-uniformly generatable, uniformly generatable) in the sense of \cite{li2024generationlenslearningtheory} (see Appendix~\ref{sec:appendix_countable} for the definitions).
\end{proposition}

\begin{proof}[Proof sketch]
Since $0<r<1$, every radius-$r$ ball in the discrete metric
is a singleton.
A positive sequence covers the support at radius $r$
if and only if it enumerates the entire support,
and each prefix covering number equals its number of
distinct observations.
Thus $r$-UUS is equivalent to every support being infinite.

Likewise, since $0<\vep_a,\vep_g\le r<1$,
the sample's $\vep_a$-covering number equals its number
of distinct observations, $\vep_a$-coverage requires
enumeration of the support, and $\vep_g$-novelty means
not having been observed.
The three generation notions therefore coincide with
their countable counterparts.
\end{proof}
\subsection{Definition of Metric Equivalence}
\label{app:equiv-metrics}
We formally define equivalent metric here for completeness sake.
\begin{definition}[Equivalent metrics]\label{def:equiv-metrics}
Let $\rho_1$ and $\rho_2$ be two metrics on the same space $\Xcal$.
We say that $\rho_1$ and $\rho_2$ are \emph{equivalent} if there exist
constants $0<c\le C<\infty$ such that for all $x,y\in\Xcal$,
\[
    c\,\rho_1(x,y) \le \rho_2(x,y) \le C\,\rho_1(x,y).
\]
\end{definition}
\subsection{Other Proofs in Section~\ref{sec:diffeps}}

This subsection collects proofs in Section~\ref{sec:diffeps}, showing how generatability properties behave under changes of scale and metric. We first establish that, in doubling metric spaces, generatability is invariant under equivalent metrics via the closure dimension. We then show that, in general metric spaces, generatability notions are monotone with respect to the novelty parameters and are stable under Lipschitz changes of the metric, with appropriate rescaling of parameters.

\begin{proof}[Proof of Theorem~\ref{thm:doubling-metric}.]
    In a doubling space, the argument for Theorem~\ref{thm:doubling-UUS} shows that UUS holds precisely when every support is nonempty and unbounded. Bi-Lipschitz equivalent metrics have the same bounded sets, and both are doubling. Hence UUS is preserved. Assume these UUS conditions hold. Choose any $\vep_a,\vep_g>0$, denote $\operatorname{C}_{\vep_a,\vep_g,1}(\Hcal)$ and $\operatorname{C}_{\vep_a,\vep_g,2}(\Hcal)$ as the $(\vep_a,\vep_g)$-closure dimension of $\Hcal$ with respect to $\rho_1$ and $\rho_2$. By Theorem~\ref{thm:unifgen} and~\ref{thm:nonunifgen}, we only need to show that $\operatorname{C}_{\vep_a,\vep_g,1}(\Hcal)<\infty$ if and only if $\operatorname{C}_{\vep_a,\vep_g,2}(\Hcal)<\infty$. Since $\Xcal$ is a doubling space and $\rho_1$ is equivalent with $\rho_2$, for any subset $A\subseteq \Xcal$,
    \[
    N(\vep_a;A,\rho_1)=\infty\iff N(\vep_a;A,\rho_2)=\infty.
    \]
    and there exist constants $C_1,C_2>0$, independent of $A$, such that whenever $N(\vep_a;A,\rho_1)<\infty$,
    \[
    C_1N(\vep_a;A,\rho_1)\le N(\vep_a;A,\rho_2)\le C_2N(\vep_a;A,\rho_1).
    \]
    For completeness, finite internal covers at the generator scale are also preserved. Choose $L>0$ such that $\rho_2\le L\rho_1$. A finite $\rho_1$-cover of a nonempty set $K$ at radius $\vep_g$, with centers in $K$, gives a finite $\rho_2$-cover at radius $L\vep_g$ with the same centers. By doubling for $\rho_2$, refine these balls to finitely many balls of radius $\vep_g/2$, and choose a point of $K$ from each smaller ball that meets $K$. These points form a finite internal $\rho_2$-cover at radius $\vep_g$. The converse follows by interchanging the metrics. Here doubling holds for both metrics because they are bi-Lipschitz equivalent. Thus the same samples have closures with finite internal covers under the two metrics, and the displayed bounds imply equivalence of finite closure dimension. Theorem~\ref{thm:unifgen} proves the uniform case, and applying the comparison to every subclass in Theorem~\ref{thm:nonunifgen} proves the non-uniform case.
\end{proof}

\begin{proof}[Proof of Theorem~\ref{thm:gendiffeps}]
Suppose $\Hcal$ is $(\vep_a,\vep_g)$-generatable in the limit, and let $\Gcal$ be an $(\vep_a,\vep_g)$-generator.  
We show that $\Gcal$ is also a $(\delta_a,\delta_g)$-generator, which implies that $\Hcal$ is $(\delta_a,\delta_g)$-generatable in the limit.

Let $\{x_i\}_{i=1}^\infty \subseteq \operatorname{supp}(h)$ be a sequence such that
\[
\operatorname{supp}(h) \subseteq B(\{x_i\}_{i=1}^\infty,\delta_a).
\]
Since $\delta_a \le \vep_a$, we also have
\[
\operatorname{supp}(h) \subseteq B\!\left(\{x_i\}_{i=1}^\infty,\vep_a\right).
\]

By the definition of an $(\vep_a,\vep_g)$-generator, there exists $s$ such that for all $t \ge s$,
\[
\Gcal(x_{1:t}) \in \operatorname{supp}(h) \setminus
B\!\left(\{x_i\}_{i=1}^t,\vep_g\right).
\]
Since $\vep_g \ge \delta_g$, it follows that
\[
\Gcal(x_{1:t}) \in \operatorname{supp}(h) \setminus
B\!\left(\{x_i\}_{i=1}^t,\delta_g\right).
\]
This completes the proof.
\end{proof}

\begin{proof}[Proof of Theorem~\ref{thm:unifdiffeps}]
We prove the result for uniform generatability; the non-uniform case follows by an analogous argument.

Suppose $\Hcal$ is $(\vep_a,\vep_g)$-uniformly generatable, and let $\Gcal$ be an $(\vep_a,\vep_g)$-uniform generator.  
Let $d \in \naturals$ be the smallest integer such that for any sequence
$\{x_i\}_{i=1}^\infty \subseteq \operatorname{supp}(h)$, if there exists $t$ with
\[
N(\vep_a;\{x_i\}_{i=1}^t,\rho) \ge d,
\]
then
\[
\Gcal(x_{1:s}) \in \operatorname{supp}(h) \setminus
B\!\left(\{x_i\}_{i=1}^s,\vep_g\right)
\quad \text{for all } s \ge t.
\]

We show that $\Gcal$ is also a $(\delta_a,\delta_g)$-uniform generator.

Fix any sequence $\{x_i\}_{i=1}^\infty \subseteq \operatorname{supp}(h)$ and suppose that for some $t$,
\[
N(\delta_a;\{x_i\}_{i=1}^t,\rho) \ge d.
\]
Since $\delta_a \ge \vep_a$, it follows that
\[
N(\vep_a;\{x_i\}_{i=1}^t,\rho) \ge d.
\]
By the definition of $\Gcal$, for all $s \ge t$ we therefore have
\[
\Gcal(x_{1:s}) \in \operatorname{supp}(h) \setminus
B\!\left(\{x_i\}_{i=1}^s,\vep_g\right).
\]
Because $\vep_g \ge \delta_g$, this implies
\[
\Gcal(x_{1:s}) \in \operatorname{supp}(h) \setminus
B\!\left(\{x_i\}_{i=1}^s,\delta_g\right).
\]
This completes the proof.
\end{proof}

\begin{proof}[Proof of Theorem~\ref{thm:diffmetricgen}]
We write $B_1(x,\vep)$ and $B_2(x,\vep)$ for the $\vep$-balls centered at $x$ with respect to $\rho_1$ and $\rho_2$, respectively.  
First observe that for any $x \in \Xcal$ and $\gamma>0$,
\[
B_1\!\left(x,\frac{\gamma}{M}\right) \subseteq B_2(x,\gamma).
\]

To prove $(r/M)$-UUS under $\rho_1$, fix $h\in\Hcal$
and a sequence $(x_i)_{i\ge1}$ in its support satisfying
\[
\operatorname{supp}(h)
\subseteq B_1(\{x_i:i\ge1\},r/M).
\]
Taking infima in $\rho_2\le M\rho_1$ gives
\[
\operatorname{supp}(h)
\subseteq B_2(\{x_i:i\ge1\},r).
\]
By $r$-UUS under $\rho_2$,
\[
N(r;\{x_1,\dots,x_t\},\rho_2)\longrightarrow\infty.
\]
Every $\rho_1$-cover at radius $r/M$ is a $\rho_2$-cover
at radius $r$, so
\[
N(r/M;\{x_1,\dots,x_t\},\rho_1)
\ge N(r;\{x_1,\dots,x_t\},\rho_2)
\longrightarrow\infty.
\]
The supports are nonempty by assumption.
Thus $\Hcal$ satisfies $(r/M)$-UUS under $\rho_1$.
Now suppose $\Gcal$ is an $(\vep_a,\vep_g)$-generator for $\Hcal$ with respect to $\rho_2$.  
We show that $\Gcal$ is also an $(\vep_a/M,\vep_g/M)$-generator with respect to $\rho_1$.

Let $\{x_i\}_{i=1}^\infty \subseteq \operatorname{supp}(h)$ be a sequence such that
\[
\operatorname{supp}(h) \subseteq B_1\!\left(\{x_i\}_{i=1}^\infty,\vep_a/M\right).
\]
For every nonempty $A\subseteq\Xcal$ and $x\in\Xcal$, taking infima in the metric inequality gives
\[
\inf_{a\in A}\rho_2(x,a)\le M\inf_{a\in A}\rho_1(x,a).
\]
Thus $B_1(A,\vep_a/M)\subseteq B_2(A,\vep_a)$ under the neighborhood definition used here, so the coverage assumption implies
\[
\operatorname{supp}(h) \subseteq B_2\!\left(\{x_i\}_{i=1}^\infty,\vep_a\right).
\]

Therefore, there exists $s$ such that for all $t\ge s$,
\[
\Gcal(x_{1:t}) \in \operatorname{supp}(h)\setminus
B_2\!\left(\{x_i\}_{i=1}^t,\vep_g\right).
\]
Since
\[
B_1\!\left(x,\vep_g/M\right) \subseteq B_2(x,\vep_g),
\]
it follows that
\[
\Gcal(x_{1:t}) \in \operatorname{supp}(h)\setminus
B_1\!\left(\{x_i\}_{i=1}^t,\vep_g/M\right).
\]
This completes the proof.
\end{proof}

\section{Remarks on Computability}
\label{sec:appendix-algo}
There are few remarks on the algorithms on generation in the limit in the countable example space setting \citep{kleinberg2024language,li2024generationlenslearningtheory,charikar2025pareto}. 
We point out that the algorithm by \cite{kleinberg2024language} for generation in the limit which only uses the oracle ``is $x\in\operatorname{supp}(h)$'' is not directly extendable to the metric space setting. Because it explicitly uses the countable structure of the example space, as well as the countable structure of the hypothesis class. We leave the extension as an open problem. The ERM-based closure-membership test in Appendix I of \cite{li2024generationlenslearningtheory} does extend to the metric setting. Finding a novel closure point requires an additional search mechanism, as discussed below.

For a generatable sample $x_1,\dots,x_n$, the closure $\langle x_1,\dots,x_n\rangle_{\Hcal}$ consists of the positive examples common to all hypotheses consistent with the sample. An Empirical Risk Minimization (ERM) oracle tests membership in this closure. Formally, let
\[
\Ocal_{\mathrm{ERM}}:
\bigl(2^{\{0,1\}^{\Xcal}}\setminus\{\emptyset\}\bigr)
\times(\Xcal\times\{0,1\})^{\star}
\longrightarrow\naturals\cup\{0\}
\]
return $\min_{h\in\Hcal}\sum_{(x,y)\in S}\mathds{1}\{h(x)\ne y\}$ on a nonempty class $\Hcal$ and a finite labeled sample $S$. Given a candidate $x$, query this oracle on the labeled sample $S_x=((x_1,1),\dots,(x_n,1),(x,0))$. Since the positive sample is realizable, the answer $r_x$ is at most $1$. It equals $0$ precisely when some hypothesis contains all observations but excludes $x$. Hence
\[
r_x=\mathds{1}\{x\in\langle x_1,\dots,x_n\rangle_{\Hcal}\}.
\]

The uniform guarantee in Lemma~\ref{lem:clossuff} can be implemented using one call per round to a max--min oracle. Fix $\vep_g>0$, a nonempty instance space $\Xcal$, and a default point $x_0\in\Xcal$. Define
\[
\mathcal{O}_{\emph{max-min}} :
\bigl(2^{\{0,1\}^{\Xcal}}\setminus\{\emptyset\}\bigr)
\times \Xcal^{\star} \to \Xcal.
\]
Given a nonempty hypothesis class $\Hcal$ and a finite sample $x_1,\dots,x_t$, the oracle returns a chosen element of
\[
\argmax_{x \in \Xcal \setminus B(\{x_i\}_{i=1}^t,\vep_g)}
\; \min_{h \in \Hcal}
\left(
\sum_{i=1}^{t} \mathds{1}\{h(x_i) \neq 1\}
+ \mathds{1}\{h(x) \neq 0\}
\right).
\]
If the feasible set is empty, the oracle returns $x_0$. For a nonempty feasible set, a maximizer exists because the objective takes values in the finite set $\{0,\dots,t+1\}$. On the empty sample, take the novelty exclusion set to be empty. These conventions define the oracle on its full domain.

The inner minimization is the ERM value defined above. Thus, whenever the feasible set is nonempty,
\[
\mathcal{O}_{\emph{max-min}}(\Hcal, x_{1:t})
\in
\argmax_{x \in \Xcal \setminus B(\{x_i\}_{i=1}^t,\vep_g)}
\;
\Ocal_{\mathrm{ERM}}\!\big(\Hcal, \{(x_1,1),\dots,(x_t,1),(x,0)\}\big).
\]

Returning the oracle's output at each round gives the uniform guarantee. To see this, fix a nonempty class as in Lemma~\ref{lem:clossuff}, a generatable sample, and suppose that
$N(\vep_a;\{x_i\}_{i=1}^t,\rho) \ge \operatorname{C}_{\vep_a,\vep_g}(\Hcal) + 1$. The closure then has no finite $\vep_g$-cover with centers in the closure. Since every observed point belongs to the closure, it follows that
\[
\langle x_1,\dots,x_t \rangle_{\Hcal}
\setminus B(\{x_i\}_{i=1}^t,\vep_g)
\neq \emptyset,
\]
and moreover, for every
$x \in \langle x_1,\dots,x_t \rangle_{\Hcal}
\setminus B(\{x_i\}_{i=1}^t,\vep_g)$,
\[
\Ocal_{\mathrm{ERM}}\!\big(\Hcal, \{(x_1,1),\dots,(x_t,1),(x,0)\}\big) \ge 1.
\]
As a result,
\[
\max_{x \in \Xcal \setminus B(\{x_i\}_{i=1}^t,\vep_g)}
\;
\Ocal_{\mathrm{ERM}}\!\big(\Hcal, \{(x_1,1),\dots,(x_t,1),(x,0)\}\big)
\ge 1,
\]
and thus the oracle $\mathcal{O}_{\emph{max-min}}(\Hcal, x_{1:t})$ returns a point
$\hat{x}_t \in \Xcal \setminus B(\{x_i\}_{i=1}^t,\vep_g)$ such that
\[
\Ocal_{\mathrm{ERM}}\!\big(\Hcal, \{(x_1,1),\dots,(x_t,1),(\hat{x}_t,0)\}\big)
\ge 1.
\]
Such a point $\hat{x}_t$ must therefore belong to
$\langle x_1,\dots,x_t \rangle_{\Hcal}$, completing the proof.

The max--min oracle supplies a search operation beyond the ERM membership test; the preceding argument does not implement it using finitely many ERM queries. For an ERM-only search, suppose a candidate enumeration $D=\{y_i:i\ge1\}$ is supplied, together with exact tests of whether $y_i\notin B(\{x_j\}_{j=1}^t,\vep_g)$. On a generatable sample, the following procedure returns a novel closure point whenever
\[
D\cap\left(\langle x_1,\dots,x_t\rangle_{\Hcal}
\setminus B(\{x_j\}_{j=1}^t,\vep_g)\right)\ne\emptyset.
\]

\begin{algorithm}[t]
\caption{Conditional search over a candidate enumeration using ERM}
\KwIn{A generatable sample $x_1,\dots,x_t$ and candidate enumeration $(y_i)_{i\ge1}$}
\For{$i=1,2,3,\dots$}{
  \If{$y_i\notin B(\{x_j\}_{j=1}^t,\vep_g)$}{
    \If{$\Ocal_{\mathrm{ERM}}\!\big(\Hcal,\{(x_1,1),\dots,(x_t,1),(y_i,0)\}\big)\ge 1$}{
      \Return{$y_i$}
    }
  }
}
\end{algorithm}

Under the displayed intersection condition, a successful candidate occurs at a finite index, so the procedure terminates after finitely many tests and ERM calls. Above the covering-number threshold in Lemma~\ref{lem:clossuff}, the novel closure is nonempty. A sufficient additional condition for the displayed intersection is that $D\cap C_t$ be dense in $C_t=\langle x_1,\dots,x_t\rangle_{\Hcal}$, since $C_t\setminus B(\{x_j\}_{j=1}^t,\vep_g)$ is relatively open in $C_t$. Density of $D$ in $\Xcal$ alone is insufficient: a dense subset can miss the entire closure. Neither the required intersection property nor a supplied candidate enumeration with that property follows from separability alone.

This conditional search need not terminate on other samples, including samples below the threshold for which no novel closure point exists. It therefore does not by itself define a generator on every finite input. A complete implementation would also need a terminating fallback rule; for example, a decidable threshold test would allow an arbitrary output below the threshold if the candidate-intersection condition were guaranteed above it. We do not assume these additional capabilities here. In general, a finite-ERM implementation per round remains unresolved by this argument, including the search for a novel closure point above the threshold.

\end{document}